%% file: main.tex
\documentclass{article} % For LaTeX2e
\usepackage{iclr2026_conference,times}

\input{math_commands.tex}

\usepackage{hyperref}
\usepackage{url}
\usepackage{graphicx}
\usepackage{xcolor}

\usepackage{amsfonts}
\usepackage{amsmath}
\usepackage{wrapfig}
\usepackage{algorithm}
\usepackage{algorithmic}
\usepackage{multirow}
\usepackage{booktabs}
\usepackage{subcaption}
\usepackage{amsthm}
\usepackage{enumitem}

\newtheorem{theorem}{Theorem}[section]

\newtheorem{lemma}[theorem]{Lemma}

\newcommand{\algo}{FlowCast}

\title{FlowCast: Trajectory Forecasting for Scalable Zero-Cost Speculative Flow Matching}

\author{
\textbf{Divya Jyoti Bajpai}\textsuperscript{1}\thanks{Work done as part of internship at Adobe Research India.},
\textbf{Shubham Agarwal}\textsuperscript{2}\thanks{Equal contribution, listed alphabetically.},
\textbf{Apoorv Saxena}\textsuperscript{3}\footnotemark[2],
\textbf{Kuldeep Kulkarni}\textsuperscript{4},\\
\textbf{Subrata Mitra}\textsuperscript{4},
\textbf{Manjesh K. Hanawal}\textsuperscript{1}
\\[0.5em]
\textsuperscript{1}Indian Institute of Technology Bombay \\
\textsuperscript{2}University of California, Berkeley \\
\textsuperscript{3}Inception Labs \\
\textsuperscript{4}Adobe Research India
}

\newif\ifcomments
\commentstrue   % set to \commentsfalse to hide all comments

\iclrfinalcopy % Uncomment for camera-ready version, but NOT for submission.
\begin{document}

\maketitle

\input{chapters/abstract}

\input{chapters/introduction}

\input{chapters/related_works}

\input{chapters/methodology}

\input{chapters/experiments}

\input{chapters/conclusion}

\input{chapters/ethics_repro}
% \newpage
\bibliography{iclr2026_conference}
\bibliographystyle{iclr2026_conference}

% \newpage

\appendix

\input{chapters/appendix}

\end{document}

%% file: math_commands.tex
\usepackage{amsmath,amsfonts,bm}

\def\eqref#1{equation~\ref{#1}}
\def\1{\bm{1}}

\DeclareMathAlphabet{\mathsfit}{\encodingdefault}{\sfdefault}{m}{sl}
\SetMathAlphabet{\mathsfit}{bold}{\encodingdefault}{\sfdefault}{bx}{n}

%% file: chapters/abstract.tex
\begin{abstract}

% Flow Matching (FM) has recently emerged as a powerful approach for high-quality visual generation. Yet, inference in FM remains prohibitively slow due to a large number of denoising steps for high-fidelity outputs, which makes real-time or interactive use impractical. Existing acceleration methods—distillation, truncation, or consistency training—either degrade quality, incur costly retraining, or lack generalization. We propose FlowCast, a training-free speculative generation framework that accelerates flow matching by extrapolating future states through constant-velocity forecasting, treating them as zero-cost drafts that are validated with a lightweight mean-squared error check. This constant-velocity forecasting allows redundant steps in stable regions to be aggressively skipped in a parallel fashion while retaining precision in complex ones. FlowCast is a plug-and-play framework that integrates seamlessly with any FM model, requires no auxiliary networks, and is backed by rigorous theoretical guarantees that bound the worst-case deviation between speculative and full FM trajectories. Empirically, \algo{} achieves $>2.5\times$ speedups in image, video, and editing tasks, outperforming existing baselines with no quality loss.

%\sm{what is "lack transferability" ?}

%\textcolor{blue}{Divya: It means generalization, changed now.}

Flow Matching (FM) has recently emerged as a powerful approach for high-quality visual generation. However, their prohibitively slow inference due to a large number of denoising steps limits their potential use in real-time or interactive applications. Existing acceleration methods, like distillation, truncation, or consistency training, either degrade quality, incur costly retraining, or lack generalization. We propose FlowCast, a training-free speculative generation framework that accelerates inference by exploiting  the fact that FM models are trained to preserve constant velocity. FlowCast speculates future velocity by extrapolating current velocity without incurring additional time cost, and accepts it if it is within a mean-squared error threshold. This constant-velocity forecasting allows redundant steps in stable regions to be aggressively skipped while retaining precision in complex ones. FlowCast is a plug-and-play framework that integrates seamlessly with any FM model and requires no auxiliary networks. We also present a theoretical analysis and bound the worst-case deviation between speculative and full FM trajectories. Empirical evaluations demonstrate that FlowCast achieves $>2.5\times$ speedup in image generation, video generation, and editing tasks, outperforming existing baselines with no quality loss as compared to standard full generation.

\end{abstract}

%% file: chapters/introduction.tex
\section{Introduction}

\textbf{Flow Matching (FM)} \cite{lipman2022flow} has recently emerged as a powerful and principled approach for generative modeling \cite{bond2021deep}. By learning a time-dependent \textit{velocity field} that continuously transforms samples from a base distribution to a target distribution, FM models generate data by integrating an \textbf{Ordinary Differential Equation (ODE)}. This framework offers advantages over diffusion-based methods \cite{croitoru2023diffusion}, including faster convergence, stable training, and direct control over sample trajectories.

However, despite these strengths, \textbf{inference in FM models is costly and is a bottleneck} \cite{kornilov2024optimal}. ODE integration is inherently sequential—each step depends on the output of the previous one—making the process slow and difficult to parallelize. For high-fidelity generation, a large number of fine-grained steps are needed, pushing the latency cost higher, which makes it difficult for real-time or interactive use. As data dimensionality and model complexity increase, so does the inference burden. This issue is even more pronounced in video generation \cite{kong2024hunyuanvideo}, where models operate on high-dimensional spatio-temporal data. The sequential ODE integration compounds across frames, making efficient inference essential for long videos \cite{liu2025timestep}.

To reduce inference latency, static reduction strategies such as \textit{distillation} \cite{luhman2021knowledge, yan2024perflow, kornilov2024optimal}, \textit{trajectory truncation} \cite{dhariwal2021diffusion, lu2022dpm, liu2025timestep}, and \textit{consistency training} \cite{yang2024consistency, zhang2025inverse, dao2025self} have been proposed. However, these approaches often degrade output quality—resulting in blurry textures, semantic drift, or broken spatial coherence. 
In the video domain, such approaches often cause temporal flickering and motion inconsistency, which severely degrade realism despite visually plausible individual frames. Furthermore, many of these methods require retraining the model with auxiliary objectives or handcrafted losses, which not only increases the engineering burden but also restricts transferability across model families. As a result, these approaches are less modular, harder to scale, and often infeasible to deploy in time-sensitive or interactive pipelines.

A promising line of work to accelerate inference in language models is \textbf{Speculative Decoding (SD)} \cite{xia2024unlocking}, where multiple drafts are proposed in parallel and selectively verified to reduce latency. However, this idea does not directly extend image and video generative models such as Flow Matching (FM). Some of the unique challenges are: (1) Unlike in language models, no readily available draft model exists. 
Then, how to design such a draft model whose trajectories closely couple with the full model? 
(2) Performance of FM is highly sensitive to velocity output. Even a small perturbation in velocity can accumulate along the integration path, leading to severe generative errors. Then how to evaluate the draft model velocity and control the trajectory drifts? We address these challenges using properties of FM trajectories. 

\textbf{Our Key Insight.} We observe that FM trajectories are typically \textit{locally smooth and slowly varying} \cite{lipman2022flow}, especially in well-trained models. Leveraging this property, we propose a speculative framework \algo{} where the model’s own past velocity predictions act as a \textit{zero-cost draft} for future steps. These drafts are selectively validated using a mean-squared error check: if accurate within the limits of a threshold, we skip recomputation entirely; if not, we fall back to the backbone. This mechanism introduces negligible overhead, requires no retraining or auxiliary models, and remains fully compatible with existing FM pipelines.

\textbf{Why it matters.} Unlike previous acceleration techniques that require retraining or trade off generation fidelity, our method is \textit{plug-and-play}, %adapting online to the complexity of the trajectory and leveraging the inherent smoothness of FM. 
adapts to the complexity of the trajectories, and leverages the inherent smoothness of the FM. 
In regions with stable dynamics, it skips steps aggressively to accelerate inference, while in complex regions it falls back to fine-grained integration for accuracy. This design enables significant speedups \textbf{without compromising output quality}. A key advantage is that the approach is entirely training-free and readily applicable to existing methods, allowing even current acceleration strategies to be further enhanced.

\algo{} is task-agnostic: it applies seamlessly to image generation, image editing, and video generation, where it ensures faithful, high-quality edits. This broad applicability highlights speculative generation as a unified framework for accelerating generative models. We compare the trajectories of the FlowCast and those obtained from the full ODE to develop a theoretical bound on their deviations. The resulting error bound guarantees the stability and fidelity of our method, ensuring that speculative generation remains close to the original FM solution while enabling faster inference.

In summary, our major contributions are:
\begin{itemize}[leftmargin=*, itemsep=0pt, topsep=0pt, parsep=0pt]
    \item We introduce a speculative generation framework viz. \algo{} for flow-matching-based generative models that enables adaptive, partially parallel inference.
    \item \algo{} uses zero-cost drafts from the model's own velocity predictions, avoiding retraining and preserving trajectory fidelity.
    \item We derive a theoretical error bound that formally characterizes how much the speculative trajectory can deviate from full ODE integration.
   % \item We provide a theoretical analysis to bound the deviation of the speculative generations from the generations of the full ODE integration. 
    \item Our method achieves substantial inference acceleration ($>2.5 \times$) while maintaining sample quality, outperforming existing baselines—especially in multi-turn editing tasks.
\end{itemize}

%% file: chapters/related_works.tex
\section{Related works}
Flow Matching \cite{lipman2022flow, dao2023flow, labs2025flux, deng2025emerging} has recently emerged as a strong alternative to diffusion models \cite{croitoru2023diffusion, xing2024survey, yang2023diffusion, zhu2023conditional}, providing a deterministic mapping between noise and data. This determinism makes flow-based models appealing for applications such as image inversion \cite{deng2024fireflow}, editing \cite{wang2024taming}, and video generation \cite{kong2024hunyuanvideo, wan2025wan}, as it reduces randomness and enables faster sampling with fewer neural function evaluations.

Multimodal extensions like FlowTok \cite{he2025flowtok} compress text and images into shared token spaces for faster inference, and FLUX.1 \cite{labs2025flux} shows large-scale flows can rival diffusion models at lower cost. For video, Pyramidal Flow Matching \cite{jin2024pyramidal} reduces compute via hierarchical generation.

However, iterative sampling remains a major bottleneck in these models, restricting real-time and interactive deployment.
To accelerate inference, prior work has focused on retraining-based solutions \cite{lee2023minimizing, bartosh2024neural}. Distillation methods \cite{song2023consistency, luhman2021knowledge, liu2022flow, kornilov2024optimal, salimans2022progressive} like InstaFlow \cite{liu2023instaflow}, LeDiFlow \cite{zwick2025lediflow} and Diff2Flow \cite{schusterbauer2025diff2flow} transfer knowledge from diffusion priors for one- or few-step sampling, while PeRFlow \cite{yan2024perflow} straightens trajectories with piecewise rectified flows. Methods used for sampling modification \cite{dhariwal2021diffusion, lu2022dpm, shaul2023bespoke} such as TeaCache \cite{liu2025timestep} skip redundant steps by getting informed by the timestep-embedding, and consistency-based approaches \cite{yang2024consistency, zhang2025inverse, haber2025iterative, dao2025self} combine adversarial and consistency losses for high-quality few-step generation.

%In contrast, our speculative decoding framework provides a novel acceleration strategy by trading off compute with speedup.
Our work accelerate the inference process using the speculative decoding framework making the FM more suitable for real-time and interactive applications.  
%Existing approaches rely on retraining or distillation, compressing the sampling process into a few learned steps at the cost of degraded fidelity. 
In contrast to the existing methods, our method do not require any retraining, distillation, trajectory straightening, or step skipping. 
%The speedup is achieved through  additional compute which is performed paralley. 
Our method offers a fully plug-and-play solution for the existing flow models. It operates dynamically at inference time by speculating future steps and verifying them on the fly. This design delivers substantial speedup, while preserving the model’s original generative quality.

%% file: chapters/methodology.tex
\section{Setup}
In this section, we provide the foundational formulation of Flow Matching (FM) and Speculative Decoding (SD) frameworks.

\subsection{Speculative Decoding Framework}
SD is a parallelizable inference strategy designed to accelerate autoregressive generation, particularly in large language models (LLMs), without retraining or sacrificing output quality. The framework leverages the observation that evaluating a sequence of tokens in parallel is only marginally slower than evaluating a single token, enabling partial parallelism during decoding.

Let \( \mathcal{M}_p \) denote the target (slow but accurate) model, and \( \mathcal{M}_q \) denote a draft model (faster but less accurate). 
%The objective is to accelerate generation from \( \mathcal{M}_p \) by using \( \mathcal{M}_q \) to generate speculative drafts and verifying them efficiently.
In SD, the objective is to accelerate token generations from \( \mathcal{M}_p \) by verifying the speculative drafts generated by \( \mathcal{M}_q \). The verification can be done at a higher speed than the generation, thus achieving speedup.  Given an input with $i$ tokens \( x_{1:i} = (x_1, x_2, \dots, x_i) \), SD has two phases:

\textbf{1) Drafting Phase:} The draft model \( \mathcal{M}_q \) generates \( K \) (speculative) tokens autoregressively:
        $(x_{i+1}^{\text{draft}}, x_{i+2}^{\text{draft}}, \dots, x_{i+K}^{\text{draft}}) \sim \mathcal{M}_q(\cdot \mid x_{1:i}).$  This phase is fast due to the small size and shallow architecture of \( \mathcal{M}_q \).

\textbf{2) Verification Phase:}
         The full model \( \mathcal{M}_p \) evaluates the entire speculative sequence in a single forward pass, computing: $\mathcal{M}_p(x_{i+k} \mid x_{1:i+k-1}^{\text{draft}})$, for $k = 1, \dots, K.$ These predicted distributions are compared against the draft tokens to assess validity (e.g., via cross-entropy or top-k agreement). Let \( j \) be the first index where \( \mathcal{M}_p \)'s prediction deviates significantly from the draft. Then the prefix \( x_{i+1}^{\text{draft}}, \dots, x_{i+j-1}^{\text{draft}} \) is accepted. Accepted tokens are then passed again to the draft model to generate new drafts based on updated context.

This speculative verification cycle continues iteratively until a full sequence is produced. The method improves throughput without modifying \( \mathcal{M}_p \), and is particularly effective when the draft model’s outputs closely align with the target model.

Under the assumption that \( \mathcal{M}_q\) is sufficiently aligned with \( \mathcal{M}_p \), the speculative decoding process preserves output quality while accelerating generation. Moreover, this approach avoids retraining and can be applied to any pair \( (\mathcal{M}_q, \mathcal{M}_p) \) where \( \mathcal{M}_q \) is reasonably well-calibrated.

Adapting speculative decoding to image generation is nontrivial, as drafting valid future states in continuous, high-dimensional latent spaces is far more complex than next-token prediction in language models. Drafts must maintain spatial and temporal consistency in predicted noise or velocity, which is difficult without a trained surrogate. Moreover, small trajectory errors can accumulate under ODE/SDE integration, causing semantic drift, structural distortions, or misalignment with conditioning. These compounding effects make speculative decoding fragile in visual domains, necessitating drafting and verification strategies that explicitly respect the geometry and stability of the underlying dynamics.

\subsection{Flow Matching Overview}\label{sec:flow_matching}
Let \( \pi_0 \) and \( \pi_1 \) be two smooth probability densities over \( \mathbb{R}^d \), representing the source and target distributions, respectively. The goal of \textbf{Flow Matching (FM)} is to construct a deterministic and continuous mapping that transports samples from \( \pi_0 \) to \( \pi_1 \) through a time-indexed trajectory governed by an ordinary differential equation (ODE). Specifically, FM defines the transformation using a time-dependent velocity field \( v : \mathbb{R}^d \times [0,1] \to \mathbb{R}^d \), yielding the following initial value problem:
\begin{equation}
\frac{dx_t}{dt} = v(x_t, t), \quad x_0 \sim \pi_0, \quad t \in [0, 1].
\end{equation}
Here, \( x_t \in \mathbb{R}^d \) denotes the state of the sample at time \( t \), and \( v(x_t,t) \) is typically parameterized by a neural network trained to align the marginal density at \( t = 1 \), denoted \( \pi_1 \), with the target distribution. The ODE defines a flow map \( \Phi_t(x_0) \) that moves points in \( \mathbb{R}^d \) along continuous trajectories, with \( \Phi_1(x_0) \sim \pi_1 \). Learning a valid \( v(x_t,t) \) that realizes this transport is the central objective in FM.

\textbf{Numerical Integration:}
In practice, the solution \( x_t \) of the ODE must be computed numerically, as closed-form solutions are generally unavailable for learned velocity fields. The most widely used method for solving the ODE in flow matching models is the \textbf{forward Euler method}, due to its simplicity and computational efficiency. The time interval \( [0,1] \) is discretized into a sequence of $K$ points \( \{ t_0 = 0, t_1, \dots, t_K = 1 \} \), which may be non-uniformly spaced. Starting from \( x_0 \sim \pi_0 \), the trajectory is updated iteratively as
\begin{equation}
x_{k+1} = x_k + \Delta t_k \cdot v(x_k, t_k), \quad \Delta t_k = t_{k+1} - t_k.
\end{equation}
This discretization approximates the continuous flow by a finite sequence of steps, with each update advancing the state in the direction prescribed by the velocity field. The Euler method is favoured in most implementations for its low computational overhead \cite{deng2025emerging, labs2025flux}.

\textbf{Trajectory Smoothness:}
The validity and accuracy of Euler-based discretization depend on the regularity of the velocity field \( v(x_t, t) \). In particular, when the ODE admits a unique solution, the discretization errors can be bounded. More importantly, empirical observations in trained FM models reveal that the velocity field tends to vary gradually across time steps and sample locations \cite{liu2025timestep}, particularly when trained using interpolants or synthetic velocity targets, indicating that the velocity does not fluctuate sharply between adjacent steps. 

% This local continuity implies that \( v(x_k, t_k) \) often provides a good approximation to \( v(x_{k+1}, t_{k+1}) \), a property we exploit in our speculative decoding framework. By leveraging this regularity, we can reuse previous velocity estimates as "drafts" for future steps, reducing redundant computation while preserving alignment with the true trajectory.

\section{Speculative Decoding for Flow models}\label{sec:SpecDec}
In this section, we present our core contributions: establishing a connection between speculative decoding and flow matching models, introducing our speculative generation framework, and providing theoretical error bounds that formally characterize its fidelity.

\textbf{Autoregressive nature of FM:} SD works well for the autoregressive models. FM process is inherently autoregressive as each denoising step depends entirely on the output of the previous step, creating a strict sequential dependency chain—much like generating tokens one by one in a language model. Though, SD is used in language tasks where the autoregression is in the discrete space (e.g., tokens from a discrete set), they can also be leveraged in FMs where autoregression is in the continuous space (e.g., velocity in a bounded interval).

\textbf{Zero-Cost Draft Construction:}
To enable speculative decoding in FM models without relying on a separate draft network, we propose a lightweight and self-contained strategy for draft generation by directly reusing the velocity predictions from the target model itself.

Recall from Section~\ref{sec:flow_matching} that FM models learn a time-dependent velocity field \( v(x_t, t) \) that transports samples from the source distribution \( \pi_0 \) toward the target distribution \( \pi_1 \) via integration of the ODE:
\begin{equation}
    \frac{d x_t}{dt} = v(x_t, t), \quad x_0 \sim \pi_0.
\end{equation}
During training, the model $v$ is typically supervised to match a reference velocity (often constant over time) along interpolated paths. This encourages the model \( v \) to vary smoothly in both space and time. Leveraging this smoothness, we propose \algo{}, where the velocity is kept constant to extrapolate not just the immediate next state, but the entire remaining trajectory, resulting in zero-cost draft trajectories. Each step of the trajectories is verified by the full model in parallel, and either accept them or apply corrections. 
%\subsection{Algorithm \algo{}}
\algo{} has three components: Drafting, Verification, and Correction. Its pseudo-code is given in Algorithm \ref{alg:algorithm}. 

\begin{wrapfigure}{l}{0.55\textwidth} % r = right, width of the box
\vspace{-0.75cm}
\centering
\begin{minipage}{0.55\textwidth} % same width as wrapfigure
\begin{algorithm}[H]
\caption{FlowCast}
\label{alg:algorithm}
\begin{algorithmic}[1]
\REQUIRE Initial state $x_{t_0}$, timesteps $\{t_k\}_{k=0}^K$, threshold $\epsilon$, model $v$
\STATE Compute initial velocity: $v_{t_0}\leftarrow v(x_{t_0}, t_0)$
\STATE Predict next state: ${x}_{t_1} \leftarrow x_{t_0} + (t_1 - t_0) v_{t_0}$
\STATE Set current index $m \gets 0$
\REPEAT
    \STATE \textbf{(Drafting)}: Generate speculative states:
     $$\tilde{x}_{t_{k+1}} \leftarrow x_{t_m} + (t_k - t_m) v_{t_m} \;\forall k= m+1,\ldots, K$$
    \STATE \textbf{(Parallel Verification)}: For each $k = m+1, \ldots, K$, compute:
    $$v_{t_k} \leftarrow v(\tilde{x}_{t_k}, t_k), \quad
    e_k \leftarrow \text{MSE}(v_{t_m}, v_{t_k})$$
    \STATE Find first index $j > m$ such that $e_j > \epsilon$
    \IF{no rejection (i.e., $e_k \leq \epsilon$ for all $k$)}
        \STATE Accept full draft $\{\tilde{x}_{t_k}\}_{k=m+1}^K$, set $m \gets K$
    \ELSE 
        \STATE \textbf{(Correction)}: Accept draft $\{\tilde{x}_{t_k}\}_{k=m+1}^{j-1}$
        \STATE ${x}_{t_j}\gets \tilde{x}_{t_{j-1}}+v_{t_{j-1}}(t_j - t_{j-1}); m \gets j-1$
        \STATE Recompute velocity: $v_{t_m} \leftarrow v(x_{t_m}, t_m)$
    \ENDIF
\UNTIL{$m = K$}
\end{algorithmic}
\end{algorithm}
\end{minipage}
\vspace{-0.5cm}
\end{wrapfigure}
\textbf{Drafting:} 
%Leveraging this smoothness, we propose a zero-cost drafting strategy where the velocity computed at time \( t_i \) is kept constant to extrapolate not just the immediate next state, but the entire remaining trajectory. Specifically, given the current state \( x_t \), we define the speculative trajectory as:
Given the current state \( x_{t_i} \) at step $t_i$, define the speculative trajectory as:
\begin{multline}\label{eq:draft_update}
    \tilde{x}_{t_i + \Delta t_k} = x_{t_i} + v(x_{t_i}, t_i)\cdot \Delta t_k 
\end{multline}
for $k = 1, \dots, n$ and $k>i$. Here \( n \) denotes the number of remaining time steps to be traversed to complete the trajectory, and $\Delta t_k=t_k-t_i$. This effectively reuses the same velocity vector to linearly extrapolate the remaining path.

This method introduces no auxiliary models or additional computational cost during drafting. While simple, it leverages the inductive bias of FM models to generate coherent trajectories, particularly in domains where velocity evolution is smooth.

\textbf{Verification:} We employ a lightweight yet principled mechanism to verify the validity of each drafted state. For every extrapolated point \( \tilde{x}_{t+\Delta t} \), we compare the reused velocity \( v(x_t, t) \) with the model’s own velocity prediction at the extrapolated point and step, i.e., \( v(\tilde{x}_{t+\Delta t}, t+\Delta t) \).
If the mean squared error (MSE) between the two remains below a pre-defined threshold $\epsilon$, the draft is accepted, i.e.,
\begin{equation}
    \text{MSE}(v(x_t, t), \, v(\tilde{x}_{t+\Delta t}, t + \Delta t)) < \epsilon.
\end{equation}
 \textbf{Correction:} If the criterion is violated at a given step, the draft is rejected, and all subsequent drafts are discarded. A new trajectory is then re-initiated starting from the rejection time step, using fresh velocity predictions from the model as per Equation~\ref{eq:draft_update}.

This verification strategy is efficient and directly aligned with the FM formulation, where the velocity field encodes the generative dynamics. Unlike traditional methods that assess validity in data or latent space, our velocity-based criterion operates in function space, making it more sensitive to inconsistencies in local dynamics while remaining inexpensive to compute. 
This ensures the integration path remains faithful to the learned flow without unnecessary recomputation.

The working procedure of \algo{} is illustrated in Figure~\ref{fig:main}. The left part shows the normal generation process, where the FM model is invoked sequentially in each step of the trajectory. The right side shows the \algo{} process for a trajectory generated from the drafting step. All the speculatively generated drafts in the trajectory are processed in parallel through the FM model. In the regular FM generation, the output of the model in each step depends only on the output of the previous step. We use this fact to verify the correctness of the drafts. In particular, if the output of the FM model at step $k$ matches that of the speculated draft at step $k+1$ within the MSE threshold, the speculated draft at $k+1$ is accepted, else rejected.
%Since each output of each model depends only on previous states, we compare the draft at $t_{k+1}$ with the model’s output at $t_{k}$ and if the mean-squared error is below a threshold, the output of the draft at $t_{k+1}$ is accepted; 
 This verification is performed efficiently in a single forward pass. In Figure~\ref{fig:main}, drafts are accepted up to step 3, but the draft at step 4 is rejected. Once a draft is rejected at a step, all later drafts are discarded, and the trajectory is computed from the last accepted step.

\begin{figure*}[!t]
    \centering
    \includegraphics[scale=0.5]{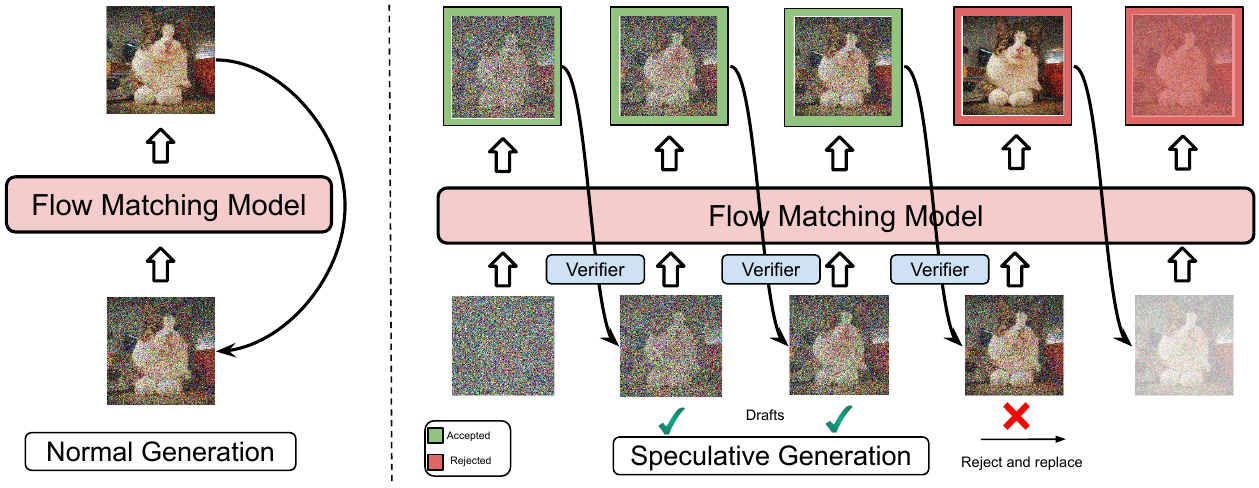}
    \caption{Comparison of normal vs. speculative image generation: The left side is the normal iterative process of generation in FM model. The right side is our approach which introduces intermediate speculative drafts that are rapidly proposed and verified in parallel by the backbone, enabling faster sampling while preserving quality.}
    \label{fig:main}
\end{figure*}

\subsection{Analysis}
In this subsection, we bound the deviations in the draft trajectories and that of the full model. We then provide a bound on the MSE threshold to guarantee the worse case deviations. 
\begin{lemma}\label{thm:lemma}
Let $x(t)$ be the solution of
$
\frac{dx}{dt} = v(x,t), x(0)=x_0,
$
where $v$ is Lipschitz in $x$ with constant $M$. Assume:  1) $\|x''(t)\|\leq N$ for all $t\in[0,1]$,  
2) $\|\tfrac{\partial v}{\partial x}(x,t)\|\leq M$ for all $(x,t)$,  
3) a forward Euler scheme with step size $h$ is run for $k$ steps, where speculative velocities $\tilde v$ are accepted in a fraction $p\in[0,1]$ of the steps and satisfy
$
\|\tilde v(x_k,t_k)-v(x_k,t_k)\| \leq \sqrt{\epsilon}.
$
Then the global discretization error $E_k = \|x(t_k)-x_k\|$ can be bounded as
\[
\|x(t_k)-x_k\| \;\leq\; \frac{e^{Mt_k}-1}{2M}\,\big(hN + 2p\sqrt{\epsilon}\big).
\]
\end{lemma}

\noindent Here, the first term $\tfrac{e^{Mt_k}-1}{2M}hN$ represents the standard Euler discretization error, while the second term $\tfrac{e^{Mt_k}-1}{M}p\sqrt{\epsilon}$ quantifies the additional deviation introduced by speculative updates. The proof of this lemma is given in Appendix \ref{sec:proof_lemma}. This lemma leads us to the theorem stated below.

\begin{theorem}\label{thm:theorem}
Let $q_d > 0$ denote a user-specified tolerance on the maximum admissible deviation from full (non-speculative)  generation trajectory. Then, for the speculative error $\|E_k^s\|$ (i.e., the deviation introduced by speculative generation) to remain bounded by $q_d$, it is sufficient to choose the speculative threshold $\epsilon$ such that
\[
\epsilon \;\leq\;\left(\frac{q_d}{2A}\right)^2,
\qquad A = \frac{e^M - 1}{M}.
\]
\end{theorem}
Using this result, we can formally control the worst-case deviation between speculative and full trajectories by setting an appropriate $\epsilon$. Theorem~\ref{thm:theorem} serves as a principled mechanism to tune $\epsilon$, ensuring that trajectory deviation through speculative steps does not exceed a quality budget. The proof of the Theorem \ref{thm:theorem} can be found in Appendix \ref{sec:proof_theorem}.

%% file: chapters/experiments.tex
\section{Experiments}
Below we provide details of our experimental setup and results. We begin with dataset descriptions. 
\begin{itemize}
    \item \textbf{Text-to-image generation}: 
We use \textbf{GenEval} dataset \cite{ghosh2023geneval}, consisting of 553 prompts targeting compositional reasoning, with categories evaluating object co-occurrence, spatial positioning, color binding, and object count.
\item \textbf{Image editing}: 
We adopt the \textbf{GEdit} benchmark \cite{liu2025step1x}, which provides 606 real-world editing instructions spanning object manipulation, color modification, etc. 
\item \textbf{Multi-turn editing}:  
We further construct an auxiliary dataset based on EditBench, where GPT-4.1 \cite{achiam2023gpt} generates three incremental edits followed by three reverse edits restoring the original image. Reconstruction fidelity is quantified via PSNR between the final output and the original.
\item \textbf{Video generation}: 
We evaluate on the \textbf{VBench} dataset \cite{huang2024vbench}, sampling 80 prompts by uniformly drawing 5 from each of the 16 evaluation dimensions, ensuring balanced coverage across motion, persistence, camera control, and scene composition.
\end{itemize}

\begin{table}[t]
\small
\centering
\footnotesize
\setlength{\tabcolsep}{6pt}
\renewcommand{\arraystretch}{1.1}
\begin{tabular}{l|ccc|ccc}
\toprule
\multirow{2}{*}{\textbf{Method}} & \multicolumn{3}{c|}{\textbf{BAGEL}} & \multicolumn{3}{c}{\textbf{FLUX}} \\
\cmidrule(lr){2-4} \cmidrule(lr){5-7}
& Overall $\uparrow$ & CLIPIQA $\uparrow$ & Speedup $\uparrow$ 
& Overall $\uparrow$ & CLIPIQA $\uparrow$ & Speedup $\uparrow$ \\
\midrule
\multicolumn{7}{c}{\textit{Full Model}} \\
\midrule
50 steps & \textbf{0.78} & \textbf{0.84} & 1.00$\times$ & \textbf{0.65} & \textbf{0.83} & 1.0$\times$ \\
25 steps & 0.77 & 0.82 & 2.0$\times$ & 0.64 & 0.80 & 2.00$\times$ \\
10 steps & 0.73 & 0.75 & 5.0$\times$ & 0.57 & 0.59 & 5.00$\times$ \\
5 steps  & 0.57 & 0.40 & 10.0$\times$ & 0.44 & 0.43 & 10.0$\times$ \\
\midrule
\multicolumn{7}{c}{\textit{Static Baselines}} \\
\midrule
InstaFlow  & 0.33 & 0.70 & 50.0$\times$ & --   & --   & -- \\
PerFlow    & 0.58 & 0.79 & 5.0$\times$  & --   & --   & -- \\
TeaCache   & 0.75 & 0.80 & 1.8$\times$  & 0.64 & 0.81 & 1.84$\times$ \\
\midrule
\multicolumn{7}{c}{\textit{Ours: Speculative Decoding}} \\
\midrule
\algo-50 & \textbf{0.78} & 0.83 & 2.5$\times$  & \textbf{0.65} & \textbf{0.83} & 2.4$\times$ \\
\algo-25 & 0.77 & 0.82 & 4.1$\times$  & 0.64 & 0.80 & 4.2$\times$ \\
\algo-10 & 0.73 & 0.74 & 7.8$\times$  & 0.57 & 0.60 & 7.6$\times$ \\
\algo-5  & 0.57 & 0.38 & 13.0$\times$ & 0.43 & 0.45 & 12.8$\times$ \\
\bottomrule
\end{tabular}
\caption{\textbf{Comparison of BAGEL and FLUX Models on Image Generation.} 
We report Overall GenEval, CLIPIQA image quality scores, and Speedup relative to the 50-step full generation. 
}
\label{tab:res_main}
% \vspace{-0.cm}
\end{table}

\textbf{Baselines:}
We compare our approach against the following baselines:

\noindent\textbf{Full Generation.} This is the standard sampling procedure, where the model executes the entire denoising trajectory without any acceleration. It provides the upper bound in terms of fidelity and serves as the reference point for all accelerated methods.
\newline
\noindent\textbf{TeaCache.} \cite{liu2025timestep} It speeds up generation by caching intermediate representations and reusing them across timesteps, guided by timestep embeddings. By eliminating computations, it achieves faster inference, at the cost of fidelity degradation due to approximation in reused states.
\newline
\noindent\textbf{InstaFlow.}\cite{liu2023instaflow} A flow-matching–based sampler trained for extremely fast generation (down to a single step) in Stable-Diffusionv-1.5 model. While efficient, it sacrifices fidelity compared to full sampling. We evaluate InstaFlow on the released Stable-Diffusion-v1.5 weights.
\newline
\noindent\textbf{PeRFlow (Piecewise Rectified Flow).} PeRFlow \cite{yan2024perflow} straightens the diffusion trajectory via piecewise-linear rectification over segmented timesteps, enabling few-step generation with strong quality–efficiency tradeoffs on Stable-Diffusion-xl model. We evaluate PerFlow on the released Stable-Diffusion-xl checkpoints.
\newline
\noindent\textbf{\algo{} (Ours).} Our approach accelerates inference by allowing a lightweight draft process to propose multiple future steps, which are then verified in parallel by the full model. We denote this as \textit{\algo-$n$}, where $n$ represents the number of speculative steps.

\begin{figure}[t]
    \centering
    
    % First subfigure
    \begin{subfigure}{0.95\textwidth}
        \centering
        \includegraphics[width=\linewidth]{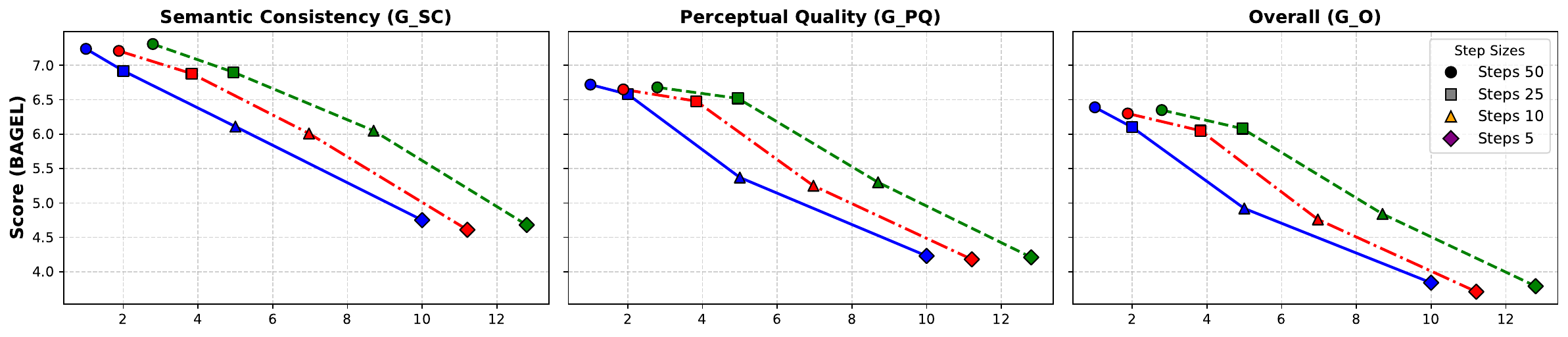}
        % \caption{Bagel Editing}
        % \label{fig:bagel_edit}
    \end{subfigure}
    
    \vspace{0.5em}

    % Second subfigure (cropped)
    \begin{subfigure}{0.95\textwidth}
        \centering
        \includegraphics[width=\linewidth]{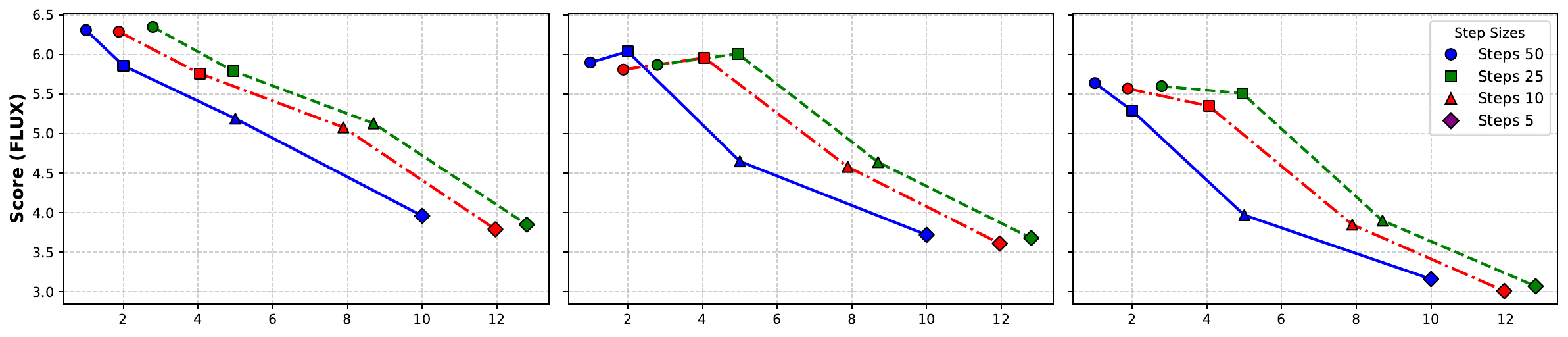}
        % \caption{Flux Editing (cropped)}
        % \label{fig:flux_edit}
    \end{subfigure}
    
    \vspace{0.5em}

    % Third subfigure
    \begin{subfigure}{0.95\textwidth}
        \centering
        \includegraphics[width=\linewidth]{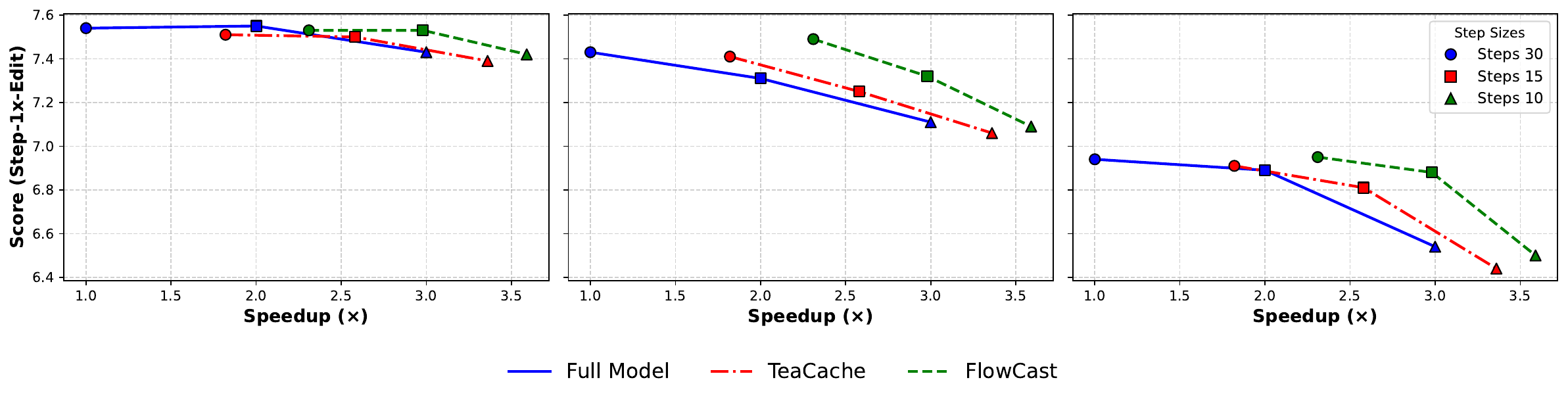}
        % \caption{Step Editing}
        % \label{fig:step_edit}
    \end{subfigure}

    % Global caption
    \caption{GEdit Scores across three models: BAGEL, FLUX, and Step-Edit. Each subfigure reports semantic consistency (G\_SC), perceptual quality (G\_PQ), and overall score (G\_O) versus speedup.}
    \label{fig:editing_quant}
\end{figure}

\begin{figure}
    \centering
    \vspace{-0.2cm}
    \includegraphics[scale = 0.25]{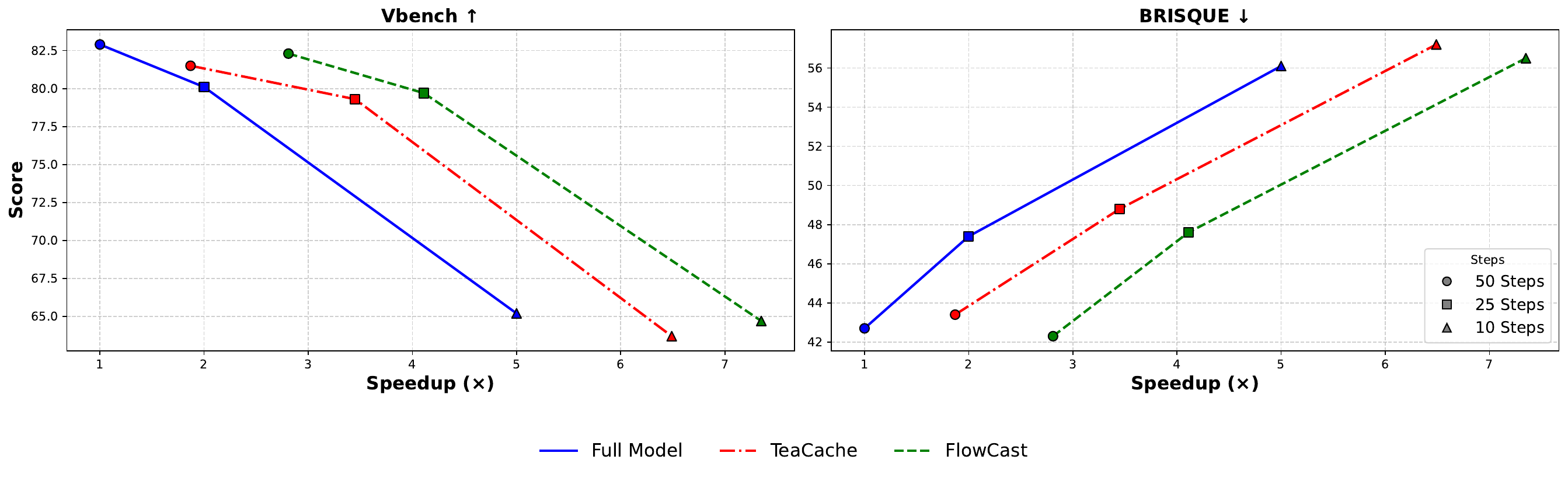}
    \caption{Results on Hunyuan Video model where we report the Vbench score and specifically for quality we use the BRISQUE metrics that assess the quality of individual frames.}
    \label{fig:video_quant}
\end{figure}
% \begin{wraptable}{r}{0.51\textwidth} % r = right, adjust width as needed
% \centering
% \small
% \begin{tabular}{ccccc}
% \hline
% Method/Metric     & G\_SC & G\_PQ & G\_O & Speedup \\ \hline
% \multicolumn{5}{c}{Full Model}                     \\ \hline
% Step-Edit-30      & 7.54  & 7.43  & 6.94 & 1.00$\times$       \\
% Step-Edit-15      & 7.55  & 7.31  & 6.89 & 2.00$\times$       \\
% Step-Edit-10      & 7.43  & 7.11  & 6.54 & 3.00$\times$       \\ \hline
% \multicolumn{5}{c}{Model+Teacache}                 \\ \hline
% Teacache-30       & 7.51  & 7.41  & 6.91 & 1.82$\times$    \\
% Teacache-15       & 7.50  & 7.25  & 6.81 & 2.58$\times$    \\
% Teacache-10       & 7.39  & 7.06  & 6.44 & 3.36$\times$    \\ \hline
% \multicolumn{5}{c}{Model+Our Method}               \\ \hline
% Spec-30           & 7.53  & 7.49  & 6.95 & 2.31$\times$    \\
% Spec-15           & 7.53  & 7.32  & 6.88 & 2.98$\times$    \\
% Spec-10           & 7.42  & 7.09  & 6.50 & 3.59$\times$    \\ \hline
% \end{tabular}
% \caption{Results of the FLUX model over the GEdit dataset for the Image Editing task.}
% \label{tab:stepedit_results}
% \end{wraptable}

Our method is complementary to existing acceleration techniques such as PeRFlow and TeaCache, further reducing redundant steps even in frameworks already optimized for efficiency. The corresponding results are reported in Tables \ref{tab:res_perflow} and \ref{tab:res_step}.

For all baselines, we adopt the official hyperparameters provided in their codebases. In Table \ref{tab:res_main}, TeaCache is applied as released, and in Figure \ref{fig:editing_quant}, we further evaluate it across different timesteps to emphasize its plug-and-play flexibility. For the choice of $\epsilon$ in \algo{}, from Theorem \ref{thm:theorem}, we have found that value of $\epsilon\in[0.01, 0.02]$ for image generation (all models), $\epsilon \in [0.07, 0.08]$ for image editing (all models), $\epsilon\in[0.001, 0.002]$ for video generation were giving high speedups with similar performance as the full model. An ablation study over $\epsilon$ across multiple steps can be found in the Appendix \ref{fig:speedupvsacc}.

\textbf{Models.} To demonstrate the versatility of our approach, we evaluate across multiple state-of-the-art models: for image generation, BAGEL \cite{deng2025emerging}, Flux-Kontext \cite{labs2025flux}, and PeRFlow \cite{yan2024perflow}; for image editing, BAGEL, Flux-Kontext, and Step-1X-Edit \cite{liu2025step1x}; and for video generation, HunyuanVideo \cite{kong2024hunyuanvideo} are used.

% \begin{wraptable}{r}{0.55\textwidth} % r = right, l = left
% \centering
% \small
% \begin{tabular}{ccccc}
% \hline
% Method/Metric & G\_SC & G\_PQ & G\_O & Speedup \\ \hline
% \multicolumn{5}{c}{Full model}                 \\ \hline
% Full 50       & 6.31  & 5.90   & 5.64 & 1.00$\times$       \\
% Full 25       & 5.86  & 6.04  & 5.29 & 2.00$\times$       \\
% Full 10       & 5.19  & 4.65  & 3.97 & 5.00$\times$       \\
% Full 5        & 3.96  & 3.72  & 3.16 & 10.00$\times$      \\ \hline
% \multicolumn{5}{c}{Speculative Decoding}       \\ \hline
% Spec 50       & 6.35  & 5.83  & 5.60  & 2.85$\times$    \\
% Spec 25       & 5.79  & 6.01  & 5.51 & 4.79$\times$    \\
% Spec 10       & 5.13  & 4.64  & 3.90  & 8.58$\times$     \\
% Spec 5        & 3.85  & 3.68  & 3.07 & 12.1$\times$    \\ \hline
% \end{tabular}
% \caption{Results of the FLUX model over the GEdit dataset for the Image Editing task.}
% \label{tab:flux_results}
% \end{wraptable}

\begin{figure}[!t]
    \centering
    \includegraphics[scale=0.61]{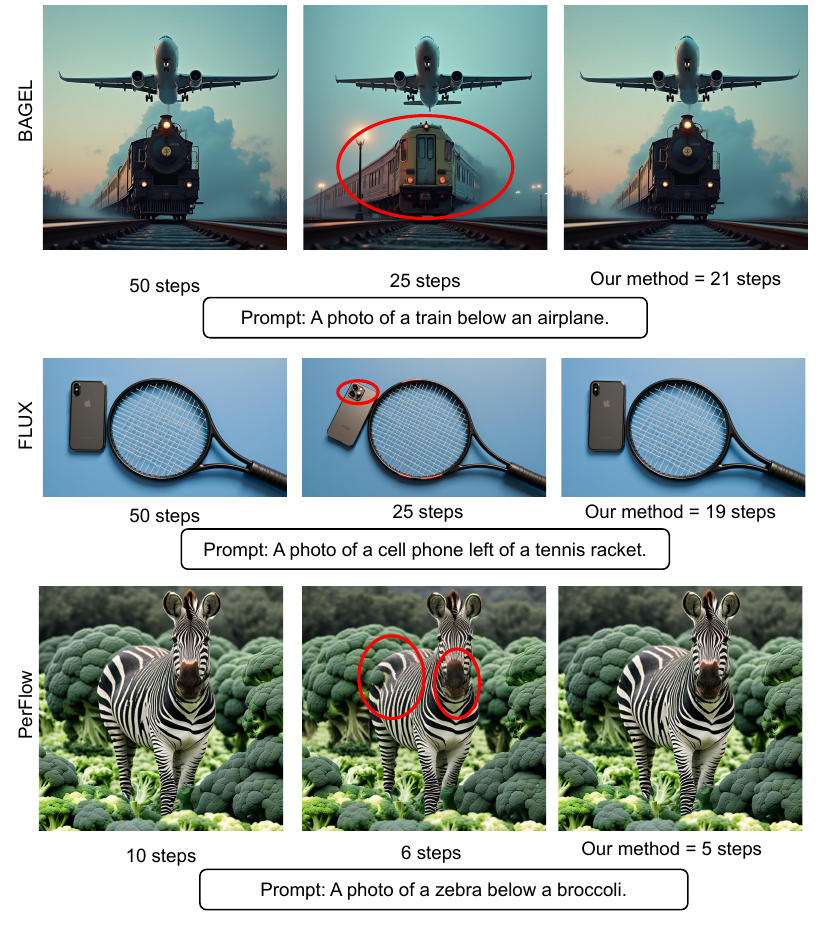}
    \caption{Speculative step reduction maintains image quality across models, whereas direct step reduction leads to noticeable degradation.}
    \label{fig:generation_examples}
\end{figure}

\subsection{Results}

\textbf{Image Generation:} Table~\ref{tab:res_main} reports results on the GenEval benchmark with BAGEL and FLUX, and Table~\ref{tab:res_perflow} shows additional evaluation on the optimized PerFlow model. GenEval scores capture semantic dimensions (detailed in Tables~\ref{tab:res_bagel}, \ref{tab:res_flux}), while fidelity is assessed via CLIPIQA. Static reduction methods (e.g., fixed step truncation) cause significant quality loss, whereas our adaptive approach achieves the lowest quality drop and preserves semantic alignment.

\textbf{Image Editing:} Figure~\ref{fig:editing_quant} presents GEdit results with BAGEL, FLUX, and Step-1X-Edit, reporting G\_SC (semantic consistency), G\_PQ (perceptual quality), and G\_O (overall editability). Static step reduction degrades both editability and quality, while our adaptive strategy consistently matches full-step generation. Multi-iteration results are shown in Figure \ref{fig:multi_iteration_quant} and discussed in Appendix \ref{sec:multi_iter}.

\textbf{Video Generation:} In Figure~\ref{fig:video_quant}, we extend our analysis to video generation using the HunyuanVideo model. Performance is reported with the vBench metric, which evaluates temporal coherence, background stability, flicker artifacts, etc. We also report the frame-wise quality score using the BRISQUE metric. Static reduction methods introduce visible temporal inconsistencies and motion artifacts, whereas \algo{} scheme maintains high-quality, temporally stable videos. 

\textbf{Qualitative analysis:}
We qualitatively compare static and adaptive step reduction across tasks. In image generation (Figure~\ref{fig:generation_examples}), static reduction often introduces semantic inconsistencies and artifacts, whereas our adaptive method skips only redundant steps, producing outputs closely aligned with full-step generations. For editing (Figures~\ref{fig:editing_examples}, \ref{fig:multiedit_example}), static reduction yields incomplete or unintended edits and compounds residual noise, while our method preserves semantic accuracy and visual fidelity—even in multi-turn settings. In video generation (Figure~\ref{fig:video_examples}), static reduction causes temporal ruptures and frame drops, whereas \algo{} maintains smoothness and fidelity.

% These results highlight the strength of our approach: steps are pruned adaptively and verified against the model’s dynamics, rather than skipped a priori. This self-verification ensures that only truly redundant updates are removed, preserving semantic and perceptual integrity. Consequently, \algo{} achieves substantial acceleration while remaining nearly indistinguishable from full-step outputs.

\textbf{Summary:} Across diverse tasks---image generation, editing, multi-iteration editing, and video generation---existing acceleration techniques consistently suffer from loss of fidelity and alignment, highlighting their limited generality. In contrast, \algo{} emerges as a universal adaptive framework for accelerating flow-based models that adapts to the generation complexity (see Appendix \ref{sec:adaptiveness}). Unlike heuristic truncation, which causes performance drops, \algo{} leverages model-informed verification to skip only redundant updates. This ensures that outputs remain \textit{ indistinguishable from full-step generation} while achieving substantial speedups. These results establish \algo{} as the preferred solution for efficient, reliable, and broadly transferable step reduction across modalities.

%% file: chapters/conclusion.tex
\section{Conclusion}
We introduced an adaptive inference strategy that accelerates generation by eliminating truly redundant computation. Our approach extends speculative decoding with zero-cost drafts obtained directly from prior velocity predictions, which are then verified in parallel through a lightweight mean-squared error criterion. Beyond its simplicity, we establish theoretical guarantees on the deviation from the original trajectory, ensuring both reliability and rigor. Extensive experiments across diverse models and tasks demonstrate that our method achieves consistent speedups without sacrificing performance, highlighting its potential as a general framework for faster and reliable inference. The main limitation of \algo{} is its dependence on parallel draft evaluations, which require adequate compute; cutting drafts reduces overhead but also reduces speedup.

\begin{figure}
    \centering
    \includegraphics[scale=0.55]{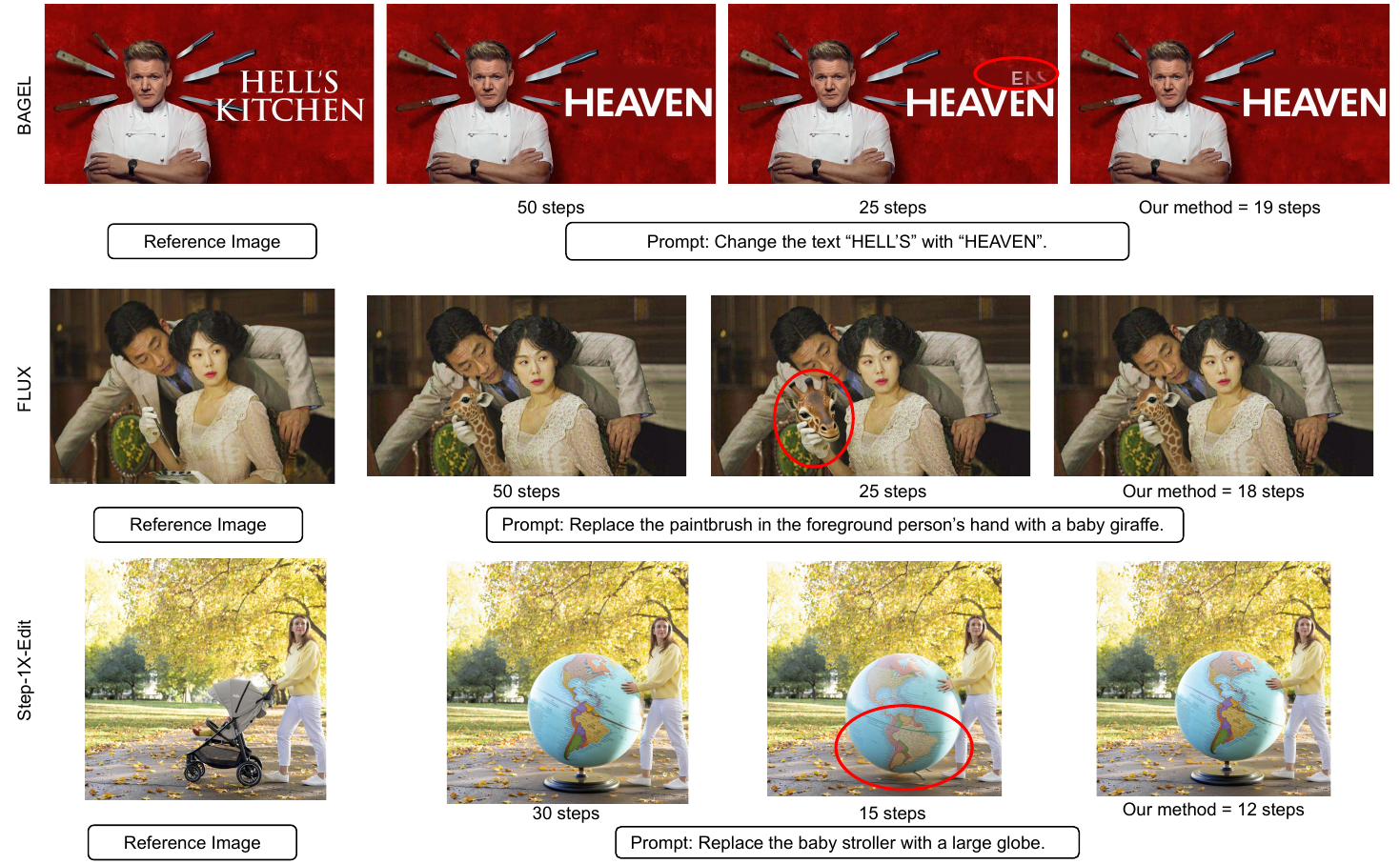}
    \caption{Speculative step reduction maintains edit fidelity and consistency better than direct step reduction.}
    \label{fig:editing_examples}
\end{figure}

% \textbf{Limitations}
% While our approach accelerates inference effectively, it assumes access to sufficient computational resources to support parallel evaluation of multiple draft trajectories. In practice, this requirement may not always hold, particularly in resource-constrained environments such as mobile devices or edge deployments, where memory and compute budgets are limited. Although our method is designed to be flexible—by monitoring the trajectory and adaptively deciding how many speculative steps to pursue—there remains an inherent trade-off: reducing the number of drafts alleviates the compute overhead but also limits the potential speedup. Consequently, our method gets its efficiency gains in scenarios where moderate parallel computing is available.

%% file: chapters/ethics_repro.tex
\section{Ethics and Reproducibility statement}
We have read and adhere to the Code of Ethics as detailed out for this conference. To the best of our knowledge, this is our original work and all related work has been appropriately cited. All authors have contributed towards this work and are responsible for it's content. There is no implicit use of Large Language Models (LLMs), except for the cases where these models are already part of our technical method, and in cases where the use has already been disclosed.

We describe a detailed algorithm to reproduce our work. We use publicly available code, models, and datasets. Additionally, we intend to provide the source code using which the results of this work can be reproduced. The results from related works are taken from their official manuscripts and/or official source repositories (as separately specified for each such work).

%% file: chapters/appendix.tex
\section{Appendix}

\subsection{Proof of Lemma \ref{thm:lemma}}\label{sec:proof_lemma}
\begin{proof}
Define the error at step \( k+1 \) as
\[
E_{k+1} := x(t_{k+1}) - x_{k+1}.
\]
Using the Euler update,
\[
x_{k+1} = x_k + h\, v^{used}(x_k, t_k),
\]
where \( v^{used} \) is either \( v_\theta \) or the speculative velocity \( \tilde{v} \) when accepted.

By Taylor expansion of \( x(t) \) around \( t_k \), there exists some \(\psi \in (t_k, t_{k+1})\) such that
\[
x(t_{k+1}) = x(t_k) + h\, x'(t_k) + \frac{h^2}{2} x''(\psi).
\]

Therefore,
\[
\begin{aligned}
\| E_{k+1} \| &= \left\| x(t_k) + h v_\theta(x(t_k), t_k) + \frac{h^2}{2} x''(\psi) - \left( x_k + h v^{used}(x_k, t_k) \right) \right\| \\
&= \left\| E_k + h \big( v_\theta(x(t_k), t_k) - v^{used}(x_k, t_k) \big) + \frac{h^2}{2} x''(\psi) \right\| \\
&\leq \| E_k \| + h \left\| v_\theta(x(t_k), t_k) - v_\theta(x_k, t_k) \right\| + h \left\| v_\theta(x_k, t_k) - v^{used}(x_k, t_k) \right\| + \frac{h^2}{2} N.
\end{aligned}
\]

By the Lipschitz continuity of \( v_\theta \) in \( x \),
\[
\left\| v_\theta(x(t_k), t_k) - v_\theta(x_k, t_k) \right\| \leq M \| x(t_k) - x_k \| = M \| E_k \|.
\]

Also, by assumption, at steps where the speculative velocity is accepted,
\[
\left\| v_\theta(x_k, t_k) - \tilde{v}(x_k, t_k) \right\| \leq \sqrt{\epsilon}.
\]
Accounting for the fraction \( p \) of such steps,
\[
\left\| v_\theta(x_k, t_k) - v^{used}(x_k, t_k) \right\| \leq p \sqrt{\epsilon}.
\]

Combining the above,
\[
\| E_{k+1} \| \leq \| E_k \| + h M \| E_k \| + h p \sqrt{\epsilon} + \frac{h^2}{2} N = (1 + h M) \| E_k \| + h p \sqrt{\epsilon} + \frac{h^2}{2} N.
\]

Let
\[
A := h p \sqrt{\epsilon} + \frac{h^2}{2} N.
\]

The recursion becomes
\[
\| E_{k+1} \| \leq (1 + h M) \| E_k \| + A.
\]

Unrolling this recursion with initial error \( E_0 = 0 \):
\[
\begin{aligned}
\| E_k \| &\leq A \sum_{i=0}^{k-1} (1 + h M)^i = A \frac{(1 + h M)^k - 1}{(1 + h M) - 1} = \frac{(1 + h M)^k - 1}{h M} A \\
&= \frac{(1 + h M)^k - 1}{2 M} \left( h N + 2 p \sqrt{\epsilon} \right).
\end{aligned}
\]

For \( h M < 1 \), it is well-known that
\[
(1 + h M)^k \leq e^{h M k} = e^{M t_k}.
\]
Hence,
\[
\| E_k \| \leq \frac{e^{M t_k} - 1}{2 M} \left( h N + 2 p \sqrt{\epsilon} \right).
\]

Since \( t_k \leq 1 \), for the entire integration interval,
\[
\| E_k \| \leq \frac{e^{M} - 1}{2 M} \left( h N + 2 p \sqrt{\epsilon} \right).
\]

\end{proof}

\subsection{Proof of Theorem \ref{thm:theorem}}\label{sec:proof_theorem}
From Lemma~\ref{thm:lemma}, we know that after full generation (i.e., at $t_k = 1$), the total error is bounded by
\[
\|E\| \;\leq\; \frac{e^M - 1}{2M}\,\big(hN + 2p\sqrt{\epsilon}\big).
\]
This upper bound can be decomposed into two contributions:
\[
\underbrace{\frac{e^M - 1}{2M}\,hN}_{\text{discretization error}} \quad+\quad \underbrace{\frac{e^M - 1}{M}\,p\sqrt{\epsilon}}_{\text{speculative deviation}}.
\]

The first term corresponds to the inherent discretization error of the full generation process, whereas the second term quantifies the additional deviation induced by speculative generation. To ensure that the speculative deviation does not exceed the admissible tolerance $q_d$, it suffices to require
\[
\frac{e^M - 1}{M}\,p\sqrt{\epsilon} \;\leq\; q_d.
\]
Rearranging, we obtain the sufficient condition
\[
\epsilon \;\leq\; \left(\frac{q_d}{2A}\right)^2, 
\qquad A = \frac{e^M - 1}{M}.
\]
Hence, the deviation introduced by speculative generation remains within the prescribed tolerance $q_d$, completing the proof.
\qed

\subsection{Multi-iteration Editing}\label{sec:multi_iter}
Figure~\ref{fig:multi_iteration_quant} evaluates the robustness of multi-step editing, where a sequence of edits is applied and subsequently reverted to reconstruct the original image. To quantify reconstruction fidelity, we report LPIPS, SSIM, and PSNR between the restored and ground-truth inputs.

The motivation for this setup is that diffusion-based editing models typically inject a small amount of noise into the image during each edit. While this noise is often imperceptible for a single edit, it can accumulate significantly when multiple sequential edits are performed. In such scenarios, naive acceleration strategies—such as uniformly reducing the number of steps—tend to amplify the compounded noise, degrading reconstruction quality.

This phenomenon is illustrated in Figure~\ref{fig:multiedit_example}. Although the full 50-step process introduces some noise, it still preserves most of the structural and semantic consistency of the original input. By contrast, a direct reduction to 25 steps produces noticeably different outputs, with the discrepancies compounding as more edits are applied.

Our method avoids this pitfall. By selectively eliminating only truly redundant steps, it achieves a substantial reduction in inference time while maintaining reconstruction quality comparable to the full model. This suggests that our approach does not compound errors across sequential edits, making it a more reliable and stable solution for practical editing workflows.

\subsection{Ablation}
In Figure \ref{fig:speedupvsacc}, we report the trade-off between speedup and reconstruction quality on the multi-turn edit dataset. The curves demonstrate that increasing the value of $\epsilon$ results in higher computational speedup but at the expense of degraded perceptual quality, as measured by PSNR. Distinct colors in the figure denote different $\epsilon$ values, thereby disentangling the contribution of this parameter to the observed trade-offs. This representation enables systematic selection of hyperparameters: for a given target speedup, one can identify the step size and $\epsilon$ that jointly yield the most favorable balance between efficiency and fidelity. For instance, as shown in Figure \ref{fig:speedupvsacc}, achieving a $3\times$ speedup can be optimally realized with 50 sampling steps and $\epsilon=0.07$, which preserves quality while significantly reducing inference cost. Thus, the figure provides empirical evidence for principled hyperparameter selection in speculative decoding under efficiency constraints.

\subsection{Adaptiveness of \algo{}}\label{sec:adaptiveness}
In Figure~\ref{fig:adaptive_example}, we highlight the adaptive nature of \algo{}. When the input prompt corresponds to a relatively simple generation, the method accepts a larger proportion of draft proposals, leading to a substantial speedup. In contrast, for more complex prompts, fewer drafts are accepted, resulting in reduced speedup but higher fidelity—reflecting the need for finer refinements in challenging cases. This trend is consistently observed across all tasks considered in our study, including image generation, image editing, and video generation. Such behavior demonstrates that \algo{} automatically tailors the trade-off between efficiency and quality to the complexity of the input, making it a flexible and broadly applicable framework for diverse generative scenarios.

\subsection{Comparison of baseline methods}\label{sec:comparison}  

\textbf{Direct Reduction.}  
A naive alternative is to simply truncate the diffusion process by reducing the number of steps. However, this \emph{static reduction} assumes that the velocity at a removed step can be approximated by reusing the previous update. Such oversimplification discards valuable intermediate information and leads to noticeable quality degradation. In contrast, \algo{} does not discard steps blindly. It adaptively identifies which steps are genuinely redundant by leveraging model outputs, ensuring faster inference without sacrificing quality.  

\textbf{TeaCache.}  
TeaCache accelerates sampling by fitting a polynomial over timestep-modulated noisy inputs to decide which steps to cache. While simple, this approach suffers from two key drawbacks:  
1) It requires model- and task-specific calibration for the polynomial fit.  
2) It is not truly adaptive --- once a target speedup is fixed, TeaCache repeatedly applies the same caching pattern (e.g., alternating steps for $2\times$ speedup), independent of prompt complexity. In our experiments (50 videos and 100 diverse image generations), we consistently observed this repetitive caching behavior, showing its limited ability to handle prompt-specific difficulty.  

\textbf{Advantages of \algo{}.}  
Unlike static approaches, \algo{} is \emph{genuinely adaptive}. It learns on-the-fly during inference, tailoring its speedup to the complexity of generation. For simpler prompts, it aggressively reduces model calls to maximize efficiency; for harder prompts, it retains more steps to safeguard fidelity. This dynamic adjustment enables \algo{} to consistently achieve a superior efficiency--quality trade-off compared to both direct reduction and TeaCache.

\subsection{Bound Tightness}
In Figure \ref{fig:bound_comparison}, we compare the theoretical error bound derived in \ref{thm:lemma} with the empirically observed error. For very small values of $\epsilon$, the theoretical bound closely matches the empirical error, indicating that the bound is tight in the low-tolerance regime. As $\epsilon$ increases, we begin to observe a slight increase in the gap between the bound and the empirical value. This behavior is expected: larger $\epsilon$ allows the algorithm to accept increasingly aggressive skips, which carry a higher risk of error accumulation.

In other words, the bound “plays safe” by overestimating potential error when the acceptance threshold is large, ensuring that the guarantee remains valid even under worst-case deviations. Thus, the slight looseness at higher $\epsilon$ is a reflection of the bound's protective nature under more error-prone conditions.

\begin{figure}
    \centering
    \includegraphics[width=0.5\linewidth]{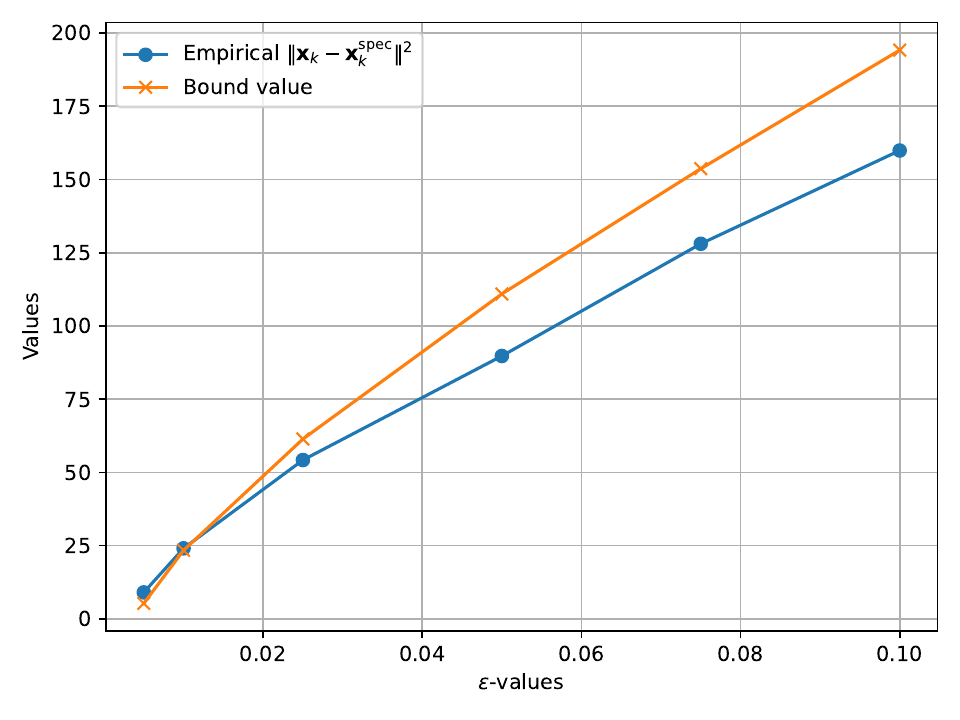}
    \caption{Comparison of the errors made by \algo{} vs the true bound value.}
    \label{fig:bound_comparison}
\end{figure}

\begin{figure}
    \centering
    \includegraphics[scale=0.279]{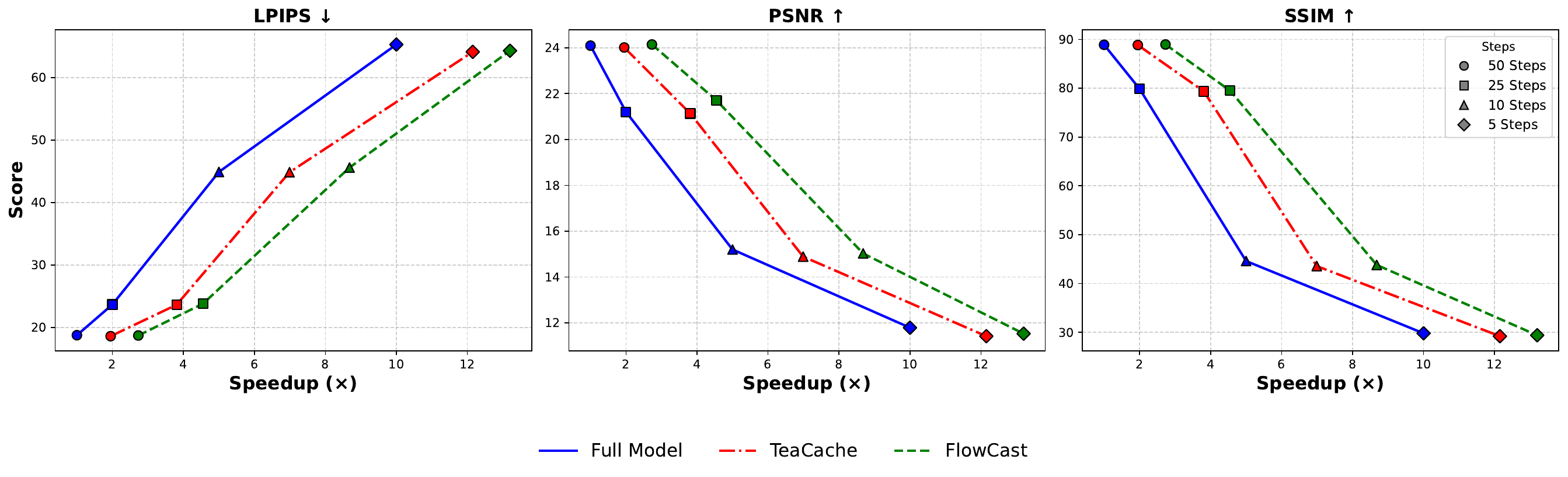}
    \caption{Results on the BAGEL model for multi-iteration editing, where the original image and the final image are compared, and LPIPS, PSNR and SSIM metrics are compared.}
    \label{fig:multi_iteration_quant}
\end{figure}

\begin{figure}
    \centering
    \includegraphics[scale=0.5]{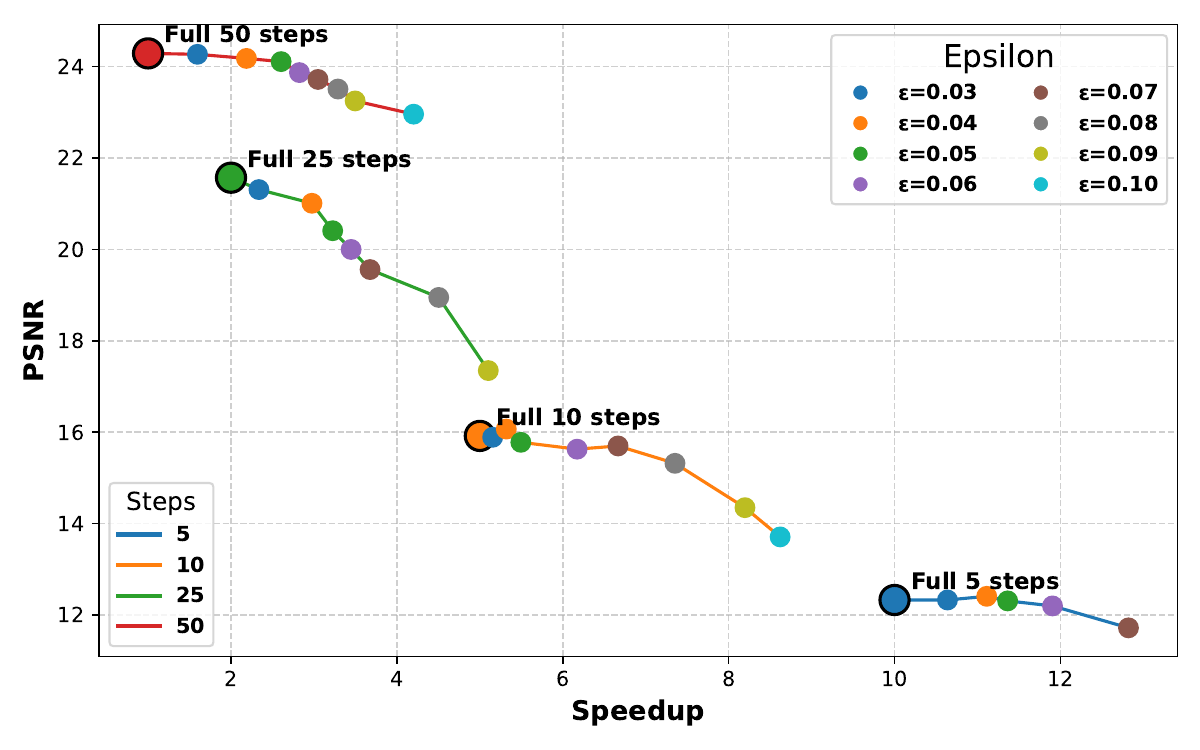}
    \caption{Speedup vs Quality curve for the multi-turn edit dataset on the BAGEL model.}
    \label{fig:speedupvsacc}
\end{figure}

\begin{figure}
    \centering
    \includegraphics[scale=0.55]{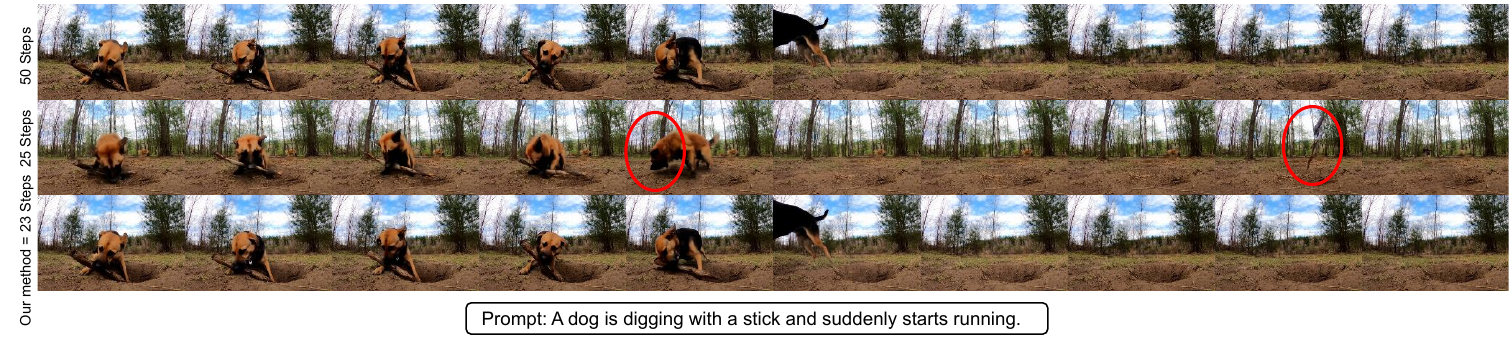}
    \caption{Speculative step reduction ensures temporal coherence and quality, outperforming direct step reduction.}
    \label{fig:video_examples}
\end{figure}

\begin{figure}
    \centering
    \includegraphics[scale=0.5]{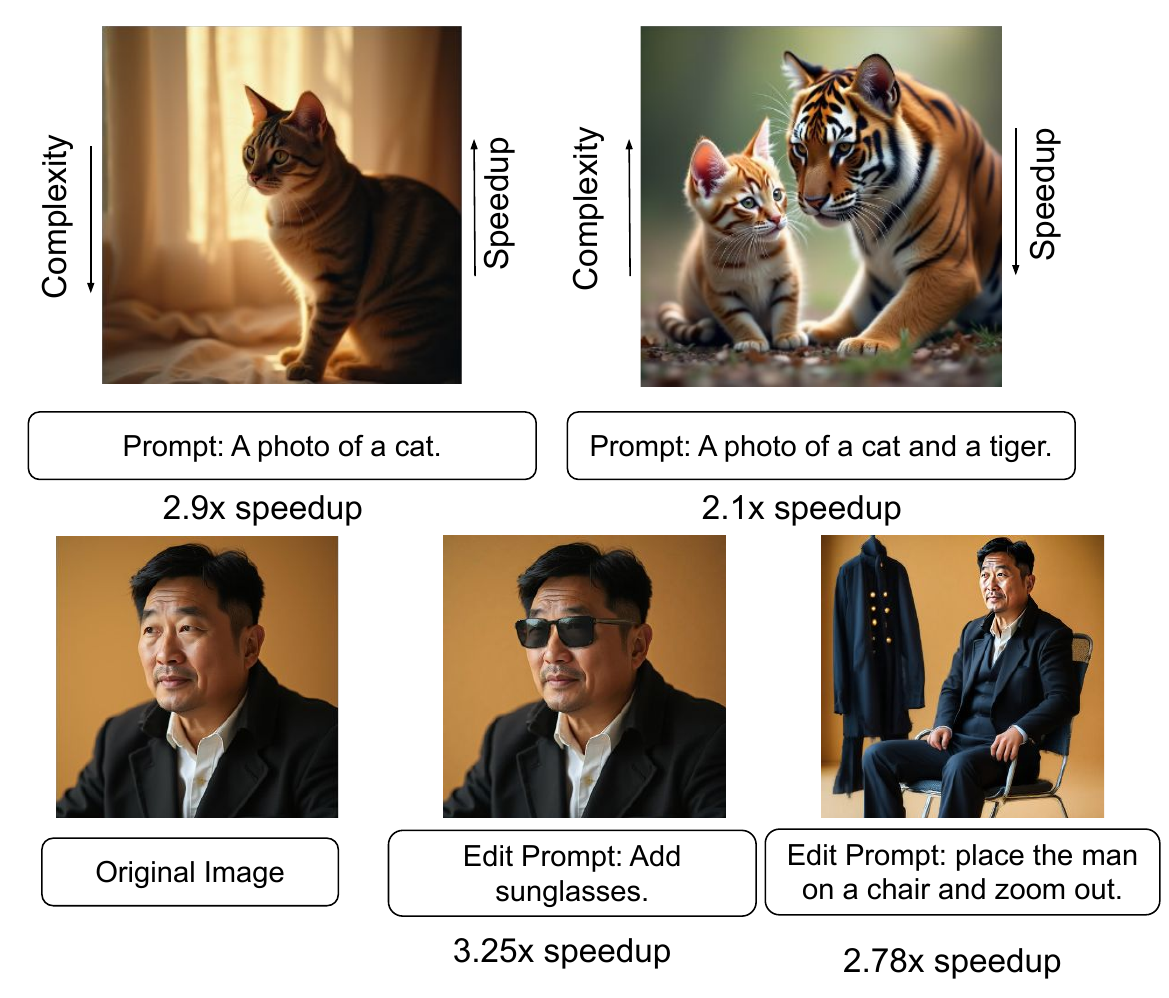}
    \caption{Examples over editing and generation tasks showing adaptiveness of our method where complex generations have lower speedups and vice-versa.}
    \label{fig:adaptive_example}
\end{figure}

\begin{figure}
    \centering
    \includegraphics[scale=0.39]{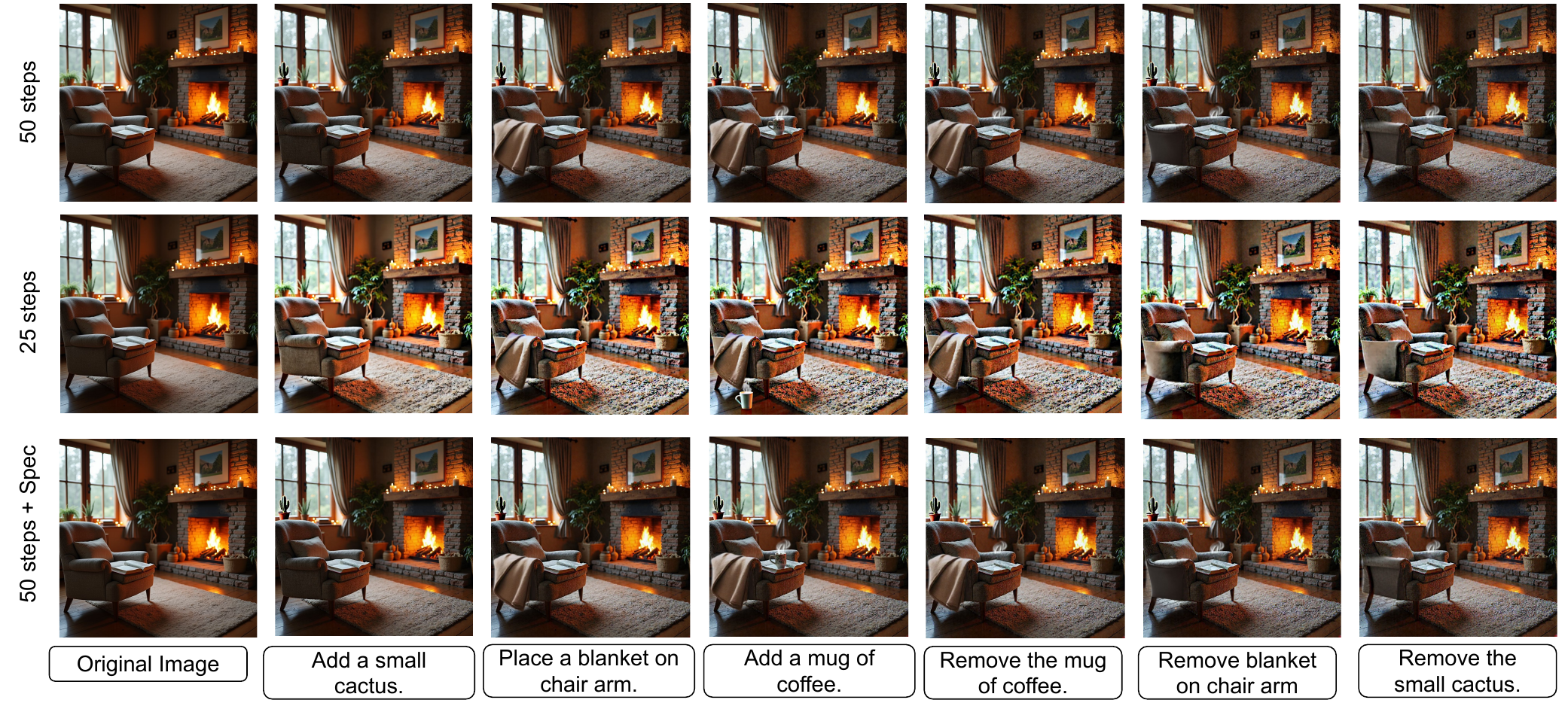}
    \caption{An instance of the multi-turn editing task where the performance of speculative generation is better than normal generation with reduced steps.}
    \label{fig:multiedit_example}
\end{figure}

\begin{figure}
    \centering
    \includegraphics[width=0.5\linewidth]{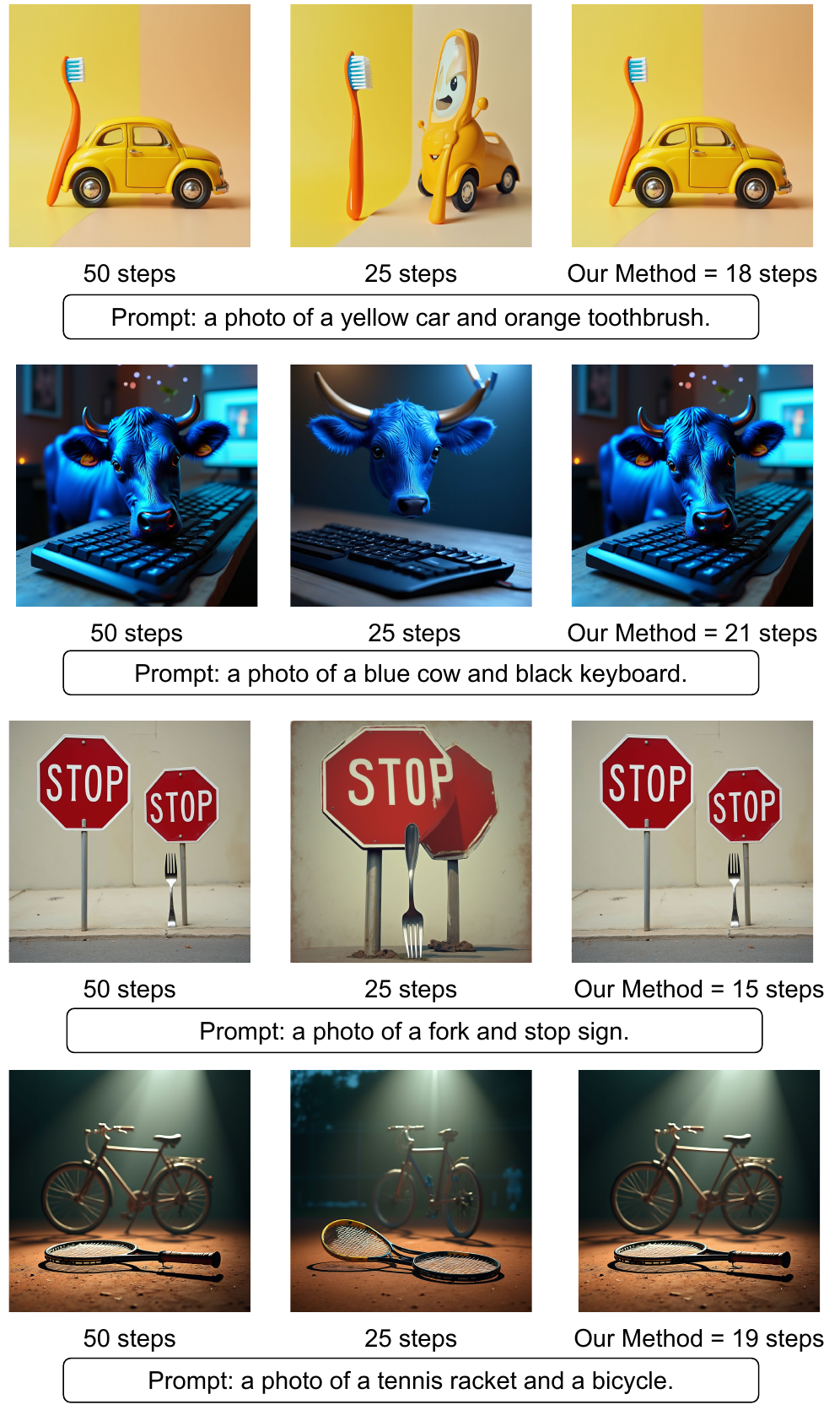}
    \caption{Multiple Generation instances using prompts.}
    \label{fig:more_generation}
\end{figure}

\begin{table}[t]
\centering
\footnotesize
\setlength{\tabcolsep}{4.5pt}
\renewcommand{\arraystretch}{1.05}
\begin{tabular}{l|cccccc|ccc}
\toprule
Method & SO & TO & CO & CL & CA & PO & Overall $\uparrow$ & CLIPIQA $\uparrow$ & Spd. $\uparrow$ \\
\midrule
\multicolumn{10}{c}{\textit{Full Model}} \\
\midrule
BAGEL-50 & 0.99 & 0.90 & 0.81 & 0.85 & 0.59 & 0.54 & \textbf{0.78} & \textbf{0.84} & 1.0$\times$ \\
BAGEL-25 & 0.99 & 0.90 & 0.78 & 0.84 & 0.62 & 0.51 & 0.77 & 0.82 & 2.0$\times$ \\
BAGEL-10 & 0.99 & 0.88 & 0.68 & 0.84 & 0.56 & 0.48 & 0.73 & 0.71 & 5.0$\times$ \\
BAGEL-5  & 0.93 & 0.62 & 0.55 & 0.72 & 0.25 & 0.36 & 0.57 & 0.40 & 10$\times$ \\
\midrule
\multicolumn{10}{c}{\textit{Static Baselines}} \\
\midrule
InstaFlow & 0.85 & 0.19 & 0.20 & 0.65 & 0.04 & 0.02 & 0.33 & 0.70 & 50$\times$ \\
PerFlow   & 0.99 & 0.79 & 0.43 & 0.85 & 0.25 & 0.15 & 0.58 & 0.79 & 5.0$\times$ \\
TeaCache  & 0.99 & 0.89 & 0.78 & 0.83 & 0.58 & 0.52 & 0.76 & 0.80 & 1.8$\times$ \\
\midrule
\multicolumn{10}{c}{\textit{Speculative Decoding (Ours)}} \\
\midrule
Spec-50 & 0.99 & 0.90 & 0.81 & 0.85 & 0.61 & 0.54 & \textbf{0.78} & 0.83 & 2.5$\times$ \\
Spec-25 & 0.99 & 0.91 & 0.77 & 0.84 & 0.62 & 0.51 & 0.77 & 0.82 & 4.1$\times$ \\
Spec-10 & 0.99 & 0.89 & 0.67 & 0.82 & 0.55 & 0.48 & 0.73 & 0.70 & 7.8$\times$ \\
Spec-5  & 0.93 & 0.61 & 0.55 & 0.73 & 0.25 & 0.36 & 0.57 & 0.38 & 13.0$\times$ \\
\bottomrule
\end{tabular}
\caption{\textbf{Results on the BAGEL model for image generation.} 
We report GenEval scores (SO: Single Object, TO: Two Objects, CO: Counting, CL: Color, CA: Category, PO: Position), overall score, CLIPIQA, and Speedup (Spd.). Speculative decoding achieves strong speedup while retaining quality close to the full model, outperforming static baselines.}
\label{tab:res_bagel}
\end{table}

\begin{table}[t]
\centering
\footnotesize
\setlength{\tabcolsep}{4.5pt}
\renewcommand{\arraystretch}{1.05}
\begin{tabular}{l|cccccc|ccc}
\toprule
Method & SO & TO & CO & CL & CA & PO & Overall $\uparrow$ & CLIPIQA $\uparrow$ & Spd. $\uparrow$ \\
\midrule
\multicolumn{10}{c}{\textit{Full Model}} \\
\midrule
FLUX-50 & 0.97 & 0.79 & 0.73 & 0.78 & 0.44 & 0.20 & \textbf{0.65} & \textbf{0.83} & 1.0$\times$ \\
FLUX-25 & 0.97 & 0.78 & 0.71 & 0.75 & 0.43 & 0.18 & 0.64 & 0.80 & 2.0$\times$ \\
FLUX-10 & 0.97 & 0.65 & 0.58 & 0.67 & 0.41 & 0.15 & 0.57 & 0.59 & 5.0$\times$ \\
FLUX-5  & 0.88 & 0.43 & 0.47 & 0.54 & 0.23 & 0.08 & 0.44 & 0.43 & 10$\times$ \\
TeaCache & 0.97 & 0.78 & 0.70 & 0.75 & 0.43 & 0.20 & 0.64 & 0.81 & 1.9$\times$ \\
\midrule
\multicolumn{10}{c}{\textit{Speculative Decoding (Ours)}} \\
\midrule
Spec-50 & 0.97 & 0.78 & 0.72 & 0.77 & 0.45 & 0.20 & \textbf{0.65} & \textbf{0.83} & 2.4$\times$ \\
Spec-25 & 0.97 & 0.78 & 0.70 & 0.76 & 0.42 & 0.18 & 0.64 & 0.80 & 4.2$\times$ \\
Spec-10 & 0.96 & 0.66 & 0.58 & 0.67 & 0.40 & 0.15 & 0.57 & 0.60 & 7.6$\times$ \\
Spec-5  & 0.87 & 0.42 & 0.46 & 0.53 & 0.24 & 0.08 & 0.43 & 0.47 & 12.8$\times$ \\
\bottomrule
\end{tabular}
\caption{\textbf{Results on the FLUX model for image generation.} 
We report GenEval scores (SO: Single Object, TO: Two Objects, CO: Counting, CL: Color, CA: Category, PO: Position), overall score, CLIPIQA, and Speedup (Spd.). Speculative decoding achieves substantial speedups while preserving performance close to the full model, surpassing static baselines.}
\label{tab:res_flux}
\end{table}

\begin{table}[t]
\centering
\footnotesize
\setlength{\tabcolsep}{4.5pt}
\renewcommand{\arraystretch}{1.05}
\begin{tabular}{l|cccccc|ccc}
\toprule
Method & SO & TO & CO & CL & CA & PO & Overall $\uparrow$ & CLIPIQA $\uparrow$ & Spd. $\uparrow$ \\
\midrule
\multicolumn{10}{c}{\textit{Full Model}} \\
\midrule
PerFlow-10 & 0.99 & 0.79 & 0.37 & 0.88 & 0.21 & 0.15 & \textbf{0.57} & \textbf{0.86} & 1.0$\times$ \\
PerFlow-6  & 0.99 & 0.78 & 0.36 & 0.86 & 0.20 & 0.15 & 0.55 & 0.83 & 2.0$\times$ \\
PerFlow-3  & 0.98 & 0.75 & 0.29 & 0.85 & 0.18 & 0.12 & 0.52 & 0.78 & 3.3$\times$ \\
\midrule
\multicolumn{10}{c}{\textit{Speculative Decoding (Ours)}} \\
\midrule
Spec-10 & 1.00 & 0.78 & 0.37 & 0.88 & 0.22 & 0.15 & \textbf{0.57} & 0.85 & 1.9$\times$ \\
Spec-6  & 0.99 & 0.77 & 0.37 & 0.87 & 0.21 & 0.15 & 0.55 & 0.82 & 2.7$\times$ \\
Spec-3  & 0.98 & 0.74 & 0.28 & 0.85 & 0.18 & 0.11 & 0.52 & 0.77 & 3.5$\times$ \\
\bottomrule
\end{tabular}
\caption{\textbf{Results on the PerFlow model for image generation.} 
We report GenEval scores (SO: Single Object, TO: Two Objects, CO: Counting, CL: Color, CA: Category, PO: Position), overall score, CLIPIQA, and Speedup (Spd.). Speculative decoding achieves strong speedups while retaining quality close to the full model.}
\label{tab:res_perflow}
\end{table}

\begin{table}[t]
\centering
\footnotesize
\setlength{\tabcolsep}{4.5pt}
\renewcommand{\arraystretch}{1.05}
\begin{tabular}{l|ccc|c}
\toprule
Method & G\_SC $\uparrow$ & G\_PQ $\uparrow$ & G\_O $\uparrow$ & Spd. $\uparrow$ \\
\midrule
\multicolumn{5}{c}{\textit{Full Model}} \\
\midrule
Step-Edit-30 & 7.54 & 7.43 & \textbf{6.94} & 1.0$\times$ \\
Step-Edit-15 & \textbf{7.55} & 7.31 & 6.89 & 2.0$\times$ \\
Step-Edit-10 & 7.43 & 7.11 & 6.54 & 3.0$\times$ \\
\midrule
\multicolumn{5}{c}{\textit{+ TeaCache}} \\
\midrule
TeaCache-30 & 7.51 & 7.41 & 6.91 & 1.8$\times$ \\
TeaCache-15 & 7.50 & 7.25 & 6.81 & 2.6$\times$ \\
TeaCache-10 & 7.39 & 7.06 & 6.44 & 3.4$\times$ \\
\midrule
\multicolumn{5}{c}{\textit{+ FlowCast (Ours)}} \\
\midrule
FlowCast-30 & 7.53 & \textbf{7.49} & \textbf{6.95} & 2.3$\times$ \\
FlowCast-15 & 7.53 & 7.32 & 6.88 & 3.0$\times$ \\
FlowCast-10 & 7.42 & 7.09 & 6.50 & 3.6$\times$ \\
\midrule
\multicolumn{5}{c}{\textit{+ TeaCache + FlowCast}} \\
\midrule
TeaCache+FlowCast-30 & 7.52 & 7.41 & 6.90 & 2.4$\times$ \\
TeaCache+FlowCast-15 & 7.48 & 7.28 & 6.85 & 3.0$\times$ \\
TeaCache+FlowCast-10 & 7.35 & 6.98 & 6.35 & 3.7$\times$ \\
\bottomrule
\end{tabular}
\caption{\textbf{Quantitative comparison of acceleration strategies.} 
We report semantic consistency (G\_SC), perceptual quality (G\_PQ), overall score (G\_O), and speedup (Spd.). FlowCast consistently accelerates inference while maintaining quality and remains complementary to TeaCache.}
\label{tab:res_step}
\end{table}

\begin{table}[t]
\centering
\begin{tabular}{lccc}
\toprule
\textbf{Method} & \textbf{Latency (s)} & \textbf{Memory (GB)} & \textbf{Speedup} \\
\midrule
Full 50            & 33.8  & 48 & 1.00$\times$ \\
Speculative 50     & 15.5  & 79 & 2.40$\times$ \\
TeaCache           & 18.3  & 48 & 1.84$\times$ \\
\midrule
Full 25            & 17.5  & 48 & 2.00$\times$ \\
Speculative 25     & 10.2  & 71 & 4.20$\times$ \\
\midrule
Full 10            & 7.1   & 48 & 5.00$\times$ \\
Speculative 10     & 5.38  & 58 & 7.60$\times$ \\
\midrule
Full 5             & 3.9   & 48 & 10.00$\times$ \\
Speculative 5      & 3.1   & 52 & 12.80$\times$ \\
\bottomrule
\end{tabular}
\caption{Memory profiling for FLUX model.}
\end{table}

\begin{table}[t]
\centering
\begin{tabular}{lccc}
\toprule
\textbf{Method} & \textbf{Latency (s)} & \textbf{Memory (GB)} & \textbf{Speedup} \\
\midrule
Full 50            & 36.8  & 29 & 1$\times$ \\
Speculative 50     & 16.1  & 65 & 2.5$\times$ \\
TeaCache           & 18.3  & 29 & 1.8$\times$ \\
\midrule
Full 25            & 19.5  & 29 & 2$\times$ \\
Speculative 25     & 10.5  & 58 & 4.1$\times$ \\
\midrule
Full 10            & 7.4   & 29 & 5$\times$ \\
Speculative 10     & 5.59  & 46 & 7.8$\times$ \\
\midrule
Full 5             & 3.8   & 29 & 10$\times$ \\
Speculative 5      & 3.2   & 36 & 13.0$\times$ \\
\bottomrule
\end{tabular}
\caption{Memory profiling for BAGEL model.}
\end{table}

%% file: main.bbl
\begin{thebibliography}{39}
\providecommand{\natexlab}[1]{#1}
\providecommand{\url}[1]{\texttt{#1}}
\expandafter\ifx\csname urlstyle\endcsname\relax
  \providecommand{\doi}[1]{doi: #1}\else
  \providecommand{\doi}{doi: \begingroup \urlstyle{rm}\Url}\fi

\bibitem[Achiam et~al.(2023)Achiam, Adler, Agarwal, Ahmad, Akkaya, Aleman, Almeida, Altenschmidt, Altman, Anadkat, et~al.]{achiam2023gpt}
Josh Achiam, Steven Adler, Sandhini Agarwal, Lama Ahmad, Ilge Akkaya, Florencia~Leoni Aleman, Diogo Almeida, Janko Altenschmidt, Sam Altman, Shyamal Anadkat, et~al.
\newblock Gpt-4 technical report.
\newblock \emph{arXiv preprint arXiv:2303.08774}, 2023.

\bibitem[Bartosh et~al.(2024)Bartosh, Vetrov, and Andersson~Naesseth]{bartosh2024neural}
Grigory Bartosh, Dmitry~P Vetrov, and Christian Andersson~Naesseth.
\newblock Neural flow diffusion models: Learnable forward process for improved diffusion modelling.
\newblock \emph{Advances in Neural Information Processing Systems}, 37:\penalty0 73952--73985, 2024.

\bibitem[Bond-Taylor et~al.(2021)Bond-Taylor, Leach, Long, and Willcocks]{bond2021deep}
Sam Bond-Taylor, Adam Leach, Yang Long, and Chris~G Willcocks.
\newblock Deep generative modelling: A comparative review of vaes, gans, normalizing flows, energy-based and autoregressive models.
\newblock \emph{IEEE transactions on pattern analysis and machine intelligence}, 44\penalty0 (11):\penalty0 7327--7347, 2021.

\bibitem[Croitoru et~al.(2023)Croitoru, Hondru, Ionescu, and Shah]{croitoru2023diffusion}
Florinel-Alin Croitoru, Vlad Hondru, Radu~Tudor Ionescu, and Mubarak Shah.
\newblock Diffusion models in vision: A survey.
\newblock \emph{IEEE transactions on pattern analysis and machine intelligence}, 45\penalty0 (9):\penalty0 10850--10869, 2023.

\bibitem[Dao et~al.(2023)Dao, Phung, Nguyen, and Tran]{dao2023flow}
Quan Dao, Hao Phung, Binh Nguyen, and Anh Tran.
\newblock Flow matching in latent space.
\newblock \emph{arXiv preprint arXiv:2307.08698}, 2023.

\bibitem[Dao et~al.(2025)Dao, Phung, Dao, Metaxas, and Tran]{dao2025self}
Quan Dao, Hao Phung, Trung~Tuan Dao, Dimitris~N Metaxas, and Anh Tran.
\newblock Self-corrected flow distillation for consistent one-step and few-step image generation.
\newblock In \emph{Proceedings of the AAAI Conference on Artificial Intelligence}, volume~39, pp.\  2654--2662, 2025.

\bibitem[Deng et~al.(2025)Deng, Zhu, Li, Gou, Li, Wang, Zhong, Yu, Nie, Song, et~al.]{deng2025emerging}
Chaorui Deng, Deyao Zhu, Kunchang Li, Chenhui Gou, Feng Li, Zeyu Wang, Shu Zhong, Weihao Yu, Xiaonan Nie, Ziang Song, et~al.
\newblock Emerging properties in unified multimodal pretraining.
\newblock \emph{arXiv preprint arXiv:2505.14683}, 2025.

\bibitem[Deng et~al.(2024)Deng, He, Mei, Wang, and Tang]{deng2024fireflow}
Yingying Deng, Xiangyu He, Changwang Mei, Peisong Wang, and Fan Tang.
\newblock Fireflow: Fast inversion of rectified flow for image semantic editing.
\newblock \emph{arXiv preprint arXiv:2412.07517}, 2024.

\bibitem[Dhariwal \& Nichol(2021)Dhariwal and Nichol]{dhariwal2021diffusion}
Prafulla Dhariwal and Alexander Nichol.
\newblock Diffusion models beat gans on image synthesis.
\newblock \emph{Advances in neural information processing systems}, 34:\penalty0 8780--8794, 2021.

\bibitem[Ghosh et~al.(2023)Ghosh, Hajishirzi, and Schmidt]{ghosh2023geneval}
Dhruba Ghosh, Hannaneh Hajishirzi, and Ludwig Schmidt.
\newblock Geneval: An object-focused framework for evaluating text-to-image alignment.
\newblock \emph{Advances in Neural Information Processing Systems}, 36:\penalty0 52132--52152, 2023.

\bibitem[Haber et~al.(2025)Haber, Ahamed, Siddiqui, Zakariaei, and Eliasof]{haber2025iterative}
Eldad Haber, Shadab Ahamed, Md~Shahriar~Rahim Siddiqui, Niloufar Zakariaei, and Moshe Eliasof.
\newblock Iterative flow matching--path correction and gradual refinement for enhanced generative modeling.
\newblock \emph{arXiv preprint arXiv:2502.16445}, 2025.

\bibitem[He et~al.(2025)He, Yu, Liu, and Chen]{he2025flowtok}
Ju~He, Qihang Yu, Qihao Liu, and Liang-Chieh Chen.
\newblock Flowtok: Flowing seamlessly across text and image tokens.
\newblock \emph{arXiv preprint arXiv:2503.10772}, 2025.

\bibitem[Huang et~al.(2024)Huang, He, Yu, Zhang, Si, Jiang, Zhang, Wu, Jin, Chanpaisit, et~al.]{huang2024vbench}
Ziqi Huang, Yinan He, Jiashuo Yu, Fan Zhang, Chenyang Si, Yuming Jiang, Yuanhan Zhang, Tianxing Wu, Qingyang Jin, Nattapol Chanpaisit, et~al.
\newblock Vbench: Comprehensive benchmark suite for video generative models.
\newblock In \emph{Proceedings of the IEEE/CVF Conference on Computer Vision and Pattern Recognition}, pp.\  21807--21818, 2024.

\bibitem[Jin et~al.(2024)Jin, Sun, Li, Xu, Jiang, Zhuang, Huang, Song, Mu, and Lin]{jin2024pyramidal}
Yang Jin, Zhicheng Sun, Ningyuan Li, Kun Xu, Hao Jiang, Nan Zhuang, Quzhe Huang, Yang Song, Yadong Mu, and Zhouchen Lin.
\newblock Pyramidal flow matching for efficient video generative modeling.
\newblock \emph{arXiv preprint arXiv:2410.05954}, 2024.

\bibitem[Kong et~al.(2024)Kong, Tian, Zhang, Min, Dai, Zhou, Xiong, Li, Wu, Zhang, et~al.]{kong2024hunyuanvideo}
Weijie Kong, Qi~Tian, Zijian Zhang, Rox Min, Zuozhuo Dai, Jin Zhou, Jiangfeng Xiong, Xin Li, Bo~Wu, Jianwei Zhang, et~al.
\newblock Hunyuanvideo: A systematic framework for large video generative models.
\newblock \emph{arXiv preprint arXiv:2412.03603}, 2024.

\bibitem[Kornilov et~al.(2024)Kornilov, Mokrov, Gasnikov, and Korotin]{kornilov2024optimal}
Nikita Kornilov, Petr Mokrov, Alexander Gasnikov, and Aleksandr Korotin.
\newblock Optimal flow matching: Learning straight trajectories in just one step.
\newblock \emph{Advances in Neural Information Processing Systems}, 37:\penalty0 104180--104204, 2024.

\bibitem[Labs et~al.(2025)Labs, Batifol, Blattmann, Boesel, Consul, Diagne, Dockhorn, English, English, Esser, et~al.]{labs2025flux}
Black~Forest Labs, Stephen Batifol, Andreas Blattmann, Frederic Boesel, Saksham Consul, Cyril Diagne, Tim Dockhorn, Jack English, Zion English, Patrick Esser, et~al.
\newblock Flux. 1 kontext: Flow matching for in-context image generation and editing in latent space.
\newblock \emph{arXiv preprint arXiv:2506.15742}, 2025.

\bibitem[Lee et~al.(2023)Lee, Kim, and Ye]{lee2023minimizing}
Sangyun Lee, Beomsu Kim, and Jong~Chul Ye.
\newblock Minimizing trajectory curvature of ode-based generative models.
\newblock In \emph{International Conference on Machine Learning}, pp.\  18957--18973. PMLR, 2023.

\bibitem[Lipman et~al.(2022)Lipman, Chen, Ben-Hamu, Nickel, and Le]{lipman2022flow}
Yaron Lipman, Ricky~TQ Chen, Heli Ben-Hamu, Maximilian Nickel, and Matt Le.
\newblock Flow matching for generative modeling.
\newblock \emph{arXiv preprint arXiv:2210.02747}, 2022.

\bibitem[Liu et~al.(2025{\natexlab{a}})Liu, Zhang, Wang, Wei, Qiu, Zhao, Zhang, Ye, and Wan]{liu2025timestep}
Feng Liu, Shiwei Zhang, Xiaofeng Wang, Yujie Wei, Haonan Qiu, Yuzhong Zhao, Yingya Zhang, Qixiang Ye, and Fang Wan.
\newblock Timestep embedding tells: It's time to cache for video diffusion model.
\newblock In \emph{Proceedings of the Computer Vision and Pattern Recognition Conference}, pp.\  7353--7363, 2025{\natexlab{a}}.

\bibitem[Liu et~al.(2025{\natexlab{b}})Liu, Han, Xing, Yin, Wang, Cheng, Liao, Wang, Fu, Han, et~al.]{liu2025step1x}
Shiyu Liu, Yucheng Han, Peng Xing, Fukun Yin, Rui Wang, Wei Cheng, Jiaqi Liao, Yingming Wang, Honghao Fu, Chunrui Han, et~al.
\newblock Step1x-edit: A practical framework for general image editing.
\newblock \emph{arXiv preprint arXiv:2504.17761}, 2025{\natexlab{b}}.

\bibitem[Liu et~al.(2022)Liu, Gong, and Liu]{liu2022flow}
Xingchao Liu, Chengyue Gong, and Qiang Liu.
\newblock Flow straight and fast: Learning to generate and transfer data with rectified flow.
\newblock \emph{arXiv preprint arXiv:2209.03003}, 2022.

\bibitem[Liu et~al.(2023)Liu, Zhang, Ma, Peng, et~al.]{liu2023instaflow}
Xingchao Liu, Xiwen Zhang, Jianzhu Ma, Jian Peng, et~al.
\newblock Instaflow: One step is enough for high-quality diffusion-based text-to-image generation.
\newblock In \emph{The Twelfth International Conference on Learning Representations}, 2023.

\bibitem[Lu et~al.(2022)Lu, Zhou, Bao, Chen, Li, and Zhu]{lu2022dpm}
Cheng Lu, Yuhao Zhou, Fan Bao, Jianfei Chen, Chongxuan Li, and Jun Zhu.
\newblock Dpm-solver: A fast ode solver for diffusion probabilistic model sampling in around 10 steps.
\newblock \emph{Advances in neural information processing systems}, 35:\penalty0 5775--5787, 2022.

\bibitem[Luhman \& Luhman(2021)Luhman and Luhman]{luhman2021knowledge}
Eric Luhman and Troy Luhman.
\newblock Knowledge distillation in iterative generative models for improved sampling speed.
\newblock \emph{arXiv preprint arXiv:2101.02388}, 2021.

\bibitem[Salimans \& Ho(2022)Salimans and Ho]{salimans2022progressive}
Tim Salimans and Jonathan Ho.
\newblock Progressive distillation for fast sampling of diffusion models.
\newblock \emph{arXiv preprint arXiv:2202.00512}, 2022.

\bibitem[Schusterbauer et~al.(2025)Schusterbauer, Gui, Fundel, and Ommer]{schusterbauer2025diff2flow}
Johannes Schusterbauer, Ming Gui, Frank Fundel, and Bj{\"o}rn Ommer.
\newblock Diff2flow: Training flow matching models via diffusion model alignment.
\newblock In \emph{Proceedings of the Computer Vision and Pattern Recognition Conference}, pp.\  28347--28357, 2025.

\bibitem[Shaul et~al.(2023)Shaul, Perez, Chen, Thabet, Pumarola, and Lipman]{shaul2023bespoke}
Neta Shaul, Juan Perez, Ricky~TQ Chen, Ali Thabet, Albert Pumarola, and Yaron Lipman.
\newblock Bespoke solvers for generative flow models.
\newblock \emph{arXiv preprint arXiv:2310.19075}, 2023.

\bibitem[Song et~al.(2023)Song, Dhariwal, Chen, and Sutskever]{song2023consistency}
Yang Song, Prafulla Dhariwal, Mark Chen, and Ilya Sutskever.
\newblock Consistency models.
\newblock 2023.

\bibitem[Wan et~al.(2025)Wan, Wang, Ai, Wen, Mao, Xie, Chen, Yu, Zhao, Yang, et~al.]{wan2025wan}
Team Wan, Ang Wang, Baole Ai, Bin Wen, Chaojie Mao, Chen-Wei Xie, Di~Chen, Feiwu Yu, Haiming Zhao, Jianxiao Yang, et~al.
\newblock Wan: Open and advanced large-scale video generative models.
\newblock \emph{arXiv preprint arXiv:2503.20314}, 2025.

\bibitem[Wang et~al.(2024)Wang, Pu, Qi, Guo, Ma, Huang, Chen, Li, and Shan]{wang2024taming}
Jiangshan Wang, Junfu Pu, Zhongang Qi, Jiayi Guo, Yue Ma, Nisha Huang, Yuxin Chen, Xiu Li, and Ying Shan.
\newblock Taming rectified flow for inversion and editing.
\newblock \emph{arXiv preprint arXiv:2411.04746}, 2024.

\bibitem[Xia et~al.(2024)Xia, Yang, Dong, Wang, Li, Ge, Liu, Li, and Sui]{xia2024unlocking}
Heming Xia, Zhe Yang, Qingxiu Dong, Peiyi Wang, Yongqi Li, Tao Ge, Tianyu Liu, Wenjie Li, and Zhifang Sui.
\newblock Unlocking efficiency in large language model inference: A comprehensive survey of speculative decoding.
\newblock \emph{arXiv preprint arXiv:2401.07851}, 2024.

\bibitem[Xing et~al.(2024)Xing, Feng, Chen, Dai, Hu, Xu, Wu, and Jiang]{xing2024survey}
Zhen Xing, Qijun Feng, Haoran Chen, Qi~Dai, Han Hu, Hang Xu, Zuxuan Wu, and Yu-Gang Jiang.
\newblock A survey on video diffusion models.
\newblock \emph{ACM Computing Surveys}, 57\penalty0 (2):\penalty0 1--42, 2024.

\bibitem[Yan et~al.(2024)Yan, Liu, Pan, Liew, Liu, and Feng]{yan2024perflow}
Hanshu Yan, Xingchao Liu, Jiachun Pan, Jun~Hao Liew, Qiang Liu, and Jiashi Feng.
\newblock Perflow: Piecewise rectified flow as universal plug-and-play accelerator.
\newblock \emph{Advances in Neural Information Processing Systems}, 37:\penalty0 78630--78652, 2024.

\bibitem[Yang et~al.(2023)Yang, Zhang, Song, Hong, Xu, Zhao, Zhang, Cui, and Yang]{yang2023diffusion}
Ling Yang, Zhilong Zhang, Yang Song, Shenda Hong, Runsheng Xu, Yue Zhao, Wentao Zhang, Bin Cui, and Ming-Hsuan Yang.
\newblock Diffusion models: A comprehensive survey of methods and applications.
\newblock \emph{ACM computing surveys}, 56\penalty0 (4):\penalty0 1--39, 2023.

\bibitem[Yang et~al.(2024)Yang, Zhang, Zhang, Liu, Xu, Zhang, Meng, Ermon, and Cui]{yang2024consistency}
Ling Yang, Zixiang Zhang, Zhilong Zhang, Xingchao Liu, Minkai Xu, Wentao Zhang, Chenlin Meng, Stefano Ermon, and Bin Cui.
\newblock Consistency flow matching: Defining straight flows with velocity consistency.
\newblock \emph{arXiv preprint arXiv:2407.02398}, 2024.

\bibitem[Zhang \& Zhou(2025)Zhang and Zhou]{zhang2025inverse}
Yuchen Zhang and Jian Zhou.
\newblock Inverse flow and consistency models.
\newblock In \emph{Forty-second International Conference on Machine Learning}, 2025.

\bibitem[Zhu et~al.(2023)Zhu, Li, Wang, He, and Yao]{zhu2023conditional}
Yuanzhi Zhu, Zhaohai Li, Tianwei Wang, Mengchao He, and Cong Yao.
\newblock Conditional text image generation with diffusion models.
\newblock In \emph{Proceedings of the IEEE/CVF Conference on Computer Vision and Pattern Recognition}, pp.\  14235--14245, 2023.

\bibitem[Zwick et~al.(2025)Zwick, Friederich, Beichter, Hilbert, Mikut, and Bringmann]{zwick2025lediflow}
Pascal Zwick, Nils Friederich, Maximilian Beichter, Lennart Hilbert, Ralf Mikut, and Oliver Bringmann.
\newblock Lediflow: Learned distribution-guided flow matching to accelerate image generation.
\newblock \emph{arXiv preprint arXiv:2505.20723}, 2025.

\end{thebibliography}
